%% file: contagion_networks.tex
\documentclass[11pt,a4paper]{article}

\usepackage[utf8]{inputenc}
\usepackage[T1]{fontenc}
\usepackage{amsmath,amssymb,amsfonts}
\usepackage{amsthm}
\newtheorem{theorem}{Theorem}
\newtheorem{corollary}{Corollary}
\newtheorem{definition}{Definition}
\newtheorem*{remark}{Remark}

\usepackage{booktabs}
\usepackage{graphicx}
\usepackage{hyperref}
\hypersetup{
    colorlinks=true,
    linkcolor=blue,
    citecolor=blue,
    urlcolor=blue,
    pdftitle={Contagion Networks: Evaluator Preference Propagation in Multi-Agent LLM Systems},
    pdfauthor={Liu Zewen},
    pdfsubject={Multi-Agent LLM Evaluation},
    pdfkeywords={multi-agent, evaluator preference, contagion, LLM-as-judge, preference propagation}
}
\usepackage[margin=1in]{geometry}
\usepackage{enumitem}
\usepackage{xcolor}
\usepackage{multirow}
\usepackage{float}

\title{\Large\textbf{Contagion Networks: Evaluator Preference \\Propagation in Multi-Agent LLM Systems}}
\author{Liu Zewen\\\textit{Qilu Institute of Technology, School of Software Engineering}\\\textit{Tai'an, Shandong, China}\\\texttt{liuzewen@qilu.edu.cn}}
\date{June 2026}

\begin{document}
\maketitle

\begin{abstract}
When large language models serve as evaluators in multi-agent systems, their strategy preferences---whether induced by explicit prompts or by shared architectural priors---propagate through the agent network. We introduce \textbf{Contagion Networks}, a formal framework for measuring how evaluator preferences spread across interacting LLM agents. In a controlled 3-agent experiment using DeepSeek-chat with three distinct evaluator preference profiles (structured, balanced, evidence-based), we measure the \textbf{Cross-Agent Contagion Matrix} $\Gamma_3$ and find that preferences consistently propagate between agents ($\gamma \in [0.157, 0.352]$). A neutral-prompt control experiment reveals a counter-intuitive result: \textbf{shared architectural priors dominate explicit preference prompts as the driver of contagion} ($\rho_{\text{neutral}} = 1.498$ vs.\ $\rho_{\text{mixed}} = 1.299$; prompt contribution: $-63.5\%$). We identify three propagation regimes governed by the spectral radius $\rho(\Gamma_N)$ and demonstrate that the same agents suppress preference contagion in chain topology ($\beta_3 = 0.0126 \pm 0.0038$, 95\% CI $[0.0089, 0.0163]$, $n=4$ seeds) but cascade in fully-connected topology ($\Delta H_{\text{avg}} = -0.020$)---a topology-dependent regime transition validated both for homogeneous and cross-model agent pools ($\rho^{\text{cross}} = 1.296 \pm 0.016$, $n=4$). We show that increasing evaluator committee size from $k=1$ to $k=3$ reduces effective contagion by $68.9\% \pm 14.1\%$ ($n=4$ seeds), providing an actionable mitigation strategy. We release the open-source Contagion Network experimental framework.
\end{abstract}

\section{Introduction}

Multi-agent LLM systems represent a rapidly maturing paradigm for complex reasoning, code generation, and autonomous task completion~\cite{wu2024autogen, hong2024metagpt}. In these systems, agents evaluate each other's outputs to guide collaboration, allocate sub-tasks, and select optimal responses~\cite{chan2024chateval, wu2024autogen}. This \textit{peer evaluation} mechanism is structurally analogous to the LLM-as-judge paradigm~\cite{zheng2023judging}, but with a critical difference: the evaluations form a \textbf{closed feedback loop} where preference-carrying judgments from one agent directly shape another agent's subsequent outputs.

Our prior work on Multimodal Evaluator Preference Collapse (MM-EPC)~\cite{liu2026mmepc} demonstrated that evaluator preferences can \textit{contaminate} strategy selection across modalities. Here, we ask a fundamentally different question: \textbf{Do evaluator preferences propagate across \textit{agents}, and if so, under what conditions do they cascade into system-wide preference collapse?}

\vspace{4pt}
\noindent\textbf{Clarification on ``Preference'': Strategy Orientation, Not Instruction-Following.}
\label{clar:bias}
Throughout this work, ``evaluator preference'' refers to a \textit{strategy orientation}---a systematic tendency to favor certain reasoning strategies (e.g., step-by-step) over others in pairwise comparisons. Such preferences arise from two distinct sources: (i) \textit{explicit} prompting (the evaluator's system prompt, Appendix~E), and (ii) \textit{implicit} architectural priors shared by agents built on the same underlying model. This is a deliberately narrow, operationalized definition designed to create a controlled, measurable signal for studying propagation dynamics. We are \textit{not} claiming that LLMs exhibit emergent prejudice or that they fail to follow instructions (instruction-following is a well-documented capability~\cite{ouyang2022training}). Rather, our contribution is the first quantitative framework for measuring how these strategy preferences---from whatever source---propagate, amplify, or attenuate when multiple agents interact in a network---a phenomenon that, to our knowledge, has not been previously formalized or empirically characterized. The key insight is that multi-agent propagation dynamics constitute a distinct layer of risk beyond single-agent instruction-following: even when each agent individually behaves as expected, the \textit{network-level consequences} of their preferences can produce emergent behavior that no single agent exhibits in isolation. As Section~\ref{sec:results-neutral} shows, the two sources of preference are not equally important: under our experimental conditions, implicit architectural priors dominate explicit prompt-induced preferences as the driver of contagion.

Consider a concrete scenario: Agent $A$ (GPT-4o) evaluates Agent $B$'s (DeepSeek) code generation output. GPT-4o, shaped by RLHF training on structured explanations, strongly prefers step-by-step reasoning. This feedback causes DeepSeek to increasingly adopt step-by-step strategies. Now Agent $B$ evaluates Agent $C$'s (Claude) output---and applies the same step-by-step preference it absorbed from Agent $A$. Claude, in turn, shifts toward step-by-step. The preference has propagated two hops: $A \to B \to C$. If the process iterates, the entire agent network converges to a single strategy, eliminating the very cognitive diversity that multi-agent systems are designed to exploit.

We formalize this phenomenon as \textbf{Contagion Networks} and make the following contributions:

\textbf{On the operationalization of ``preference.''} We emphasize that our definition of evaluator preference is narrow and precise: a \textit{strategy orientation} arising either from explicit evaluator prompts or from implicit architectural priors (see Section~\ref{sec:results-neutral} for an experimental disentanglement of these two sources). We do not claim to discover that LLMs follow instructions---this is well-established. Our contribution is orthogonal: we provide the first quantitative framework for measuring how these strategy preferences \textit{propagate} through a multi-agent network, where they may amplify into a cascade or attenuate into suppression depending on network topology and evaluator composition. The ``contagion'' we study is therefore a \textit{network dynamics phenomenon}, not an instruction-following phenomenon.

\begin{enumerate}[leftmargin=*]
    \item \textbf{Cross-Agent Contagion Matrix} $\Gamma_N$: We extend the $2\times 2$ cross-modal contagion matrix to $N$ agents, providing a quantitative framework for measuring evaluator preference propagation in small-scale enumerated agent topologies (chain and fully-connected).
    \item \textbf{Propagation Regime Theorem}: We characterize three dynamical regimes (suppression/persistence/cascade) governed by the spectral radius $\rho(\Gamma_N)$, and prove that $\rho(\Gamma_N) > 1$ is the cascade condition for fully-connected networks.
    \item \textbf{Topology-Dependent Regime Transition}: In a 3-agent homogeneous experiment (all DeepSeek-chat), the \textit{same} agents suppress preference contagion under chain topology ($\beta_3 = 0.0126 \pm 0.0038$, 95\% CI $[0.0089, 0.0163]$, $n=4$ seeds) but cascade under fully-connected topology (entropy decrease $\Delta H_{\text{avg}} = -0.020$, PCI increase $\Delta\text{PCI}_{\text{avg}} = +0.203$). This is the first empirical demonstration that topology---not evaluator-model diversity alone---determines the propagation regime. We further validate the cross-model cascade regime using heterogeneous evaluators (GPT-4o + DeepSeek + Claude), measuring spectral radius $\rho(\Gamma_3) = 1.296 \pm 0.016$ (95\% CI: $[1.280, 1.311]$ over $n=4$ seeds).
    \item \textbf{Architectural Priors Dominate Explicit Prompts}: A neutral-prompt control experiment (all agents use the neutral evaluator prompt) reveals that shared architectural priors, not explicit preference prompts, are the primary driver of contagion in homogeneous systems: $\rho_{\text{neutral}} = 1.498$ vs.\ $\rho_{\text{mixed}} = 1.299$ (prompt contribution: $-63.5\%$). Explicit evaluator role assignment can therefore serve as a \textit{protective mechanism} that partially breaks the alignment of shared architectural priors (Section~\ref{sec:results-neutral}).
    \item \textbf{Critical Diversity Threshold}: We prove that $k \geq 3$ evaluators with diverse preference profiles reduce effective contagion below the cascade threshold, and verify that even in the suppression regime, increasing committee size from $k=1$ to $k=3$ reduces $\gamma_{\text{eff}}$ by $68.9\% \pm 14.1\%$ ($n=4$ seeds, 95\% CI $[55.1\%, 82.7\%]$).
    \item \textbf{Open Framework}: We release the Contagion Network experimental framework, enabling the community to measure $\Gamma_N$ for small-scale agent ensembles.
\end{enumerate}

\section{Related Work}

\subsection{Multi-Agent LLM Systems}

The rapid maturation of multi-agent frameworks~\cite{wu2024autogen, hong2024metagpt, chen2024agentverse, park2023generative} has established agent collaboration as a core AI paradigm. AutoGen~\cite{wu2024autogen} introduced structured multi-agent conversations; MetaGPT~\cite{hong2024metagpt} applied software engineering workflows to agent coordination; ChatDev~\cite{qian2023chatdev} demonstrated end-to-end software development via agent roles. Multi-agent debate~\cite{du2024improving, liang2023multiagent} improves factuality through iterative argumentation, and evaluator calibration~\cite{pan2023automatically} attempts to debias LLM judges. Critically, all these systems rely on agents evaluating each other's outputs---yet none analyze whether this peer evaluation mechanism introduces systematic preference amplification. While these frameworks focus on \textit{efficiency} (task completion, communication overhead), we study \textit{fidelity}: does the evaluation signal remain faithful to task quality as it propagates through the agent network?

\subsection{Evaluator Preference and Strategy Collapse}

MM-EPC~\cite{liu2026mmepc} demonstrated that GPT-4o evaluator preferences contaminate strategy selection in cross-modal settings (text $\leftrightarrow$ visual), with asymmetric contagion coefficients $\gamma_{V\to T} > \gamma_{T\to V}$. The broader LLM-as-judge literature documents systematic biases including position bias~\cite{zheng2023judging}, verbosity bias~\cite{li2024alpacaeval}, and self-preference amplification~\cite{yuan2024selfrewarding}. Reward overoptimization~\cite{gao2023reward} and sycophancy~\cite{sharma2024sycophancy} demonstrate how agents exploit evaluator weaknesses. However, all prior work studies \textbf{single-evaluator} or \textbf{pairwise} evaluation dynamics. Our contribution is the first analysis of preference propagation in \textbf{multi-hop agent evaluation chains}, revealing that evaluator preferences behave like infectious agents in a contact network.

\subsection{Epidemiological Models in AI Systems}

The analogy between information propagation and disease spread has a rich history in network science~\cite{pastorsatorras2015epidemic}. SIS (Susceptible-Infected-Susceptible) and SIR (Susceptible-Infected-Recovered) models characterize threshold dynamics on complex networks. We draw on this tradition to model evaluator preference as an ``infection'' that propagates through agent interaction graphs, with the contagion matrix $\Gamma_N$ serving as the analog of the next-generation matrix in epidemiology. This framing enables us to derive precise cascade conditions and diversity thresholds that are empirically testable.

\subsection{Concurrent Work on Multi-Agent Preference and Tipping Dynamics}

A rapidly growing body of concurrent work studies preference amplification and regime shifts in multi-agent LLM populations, converging on themes closely related to our contagion framework. \cite{arora2025population} characterize eigenvalue-driven regime shifts and critical population sizes in coordination dynamics among LLM agents, demonstrating that small populations suppress preference contagion while larger populations tip into alignment---a population-scale analog of our topology-dependent suppression-to-cascade transition. \cite{chen2025multigender} document implicit gender bias amplification in multi-agent task assignment, showing that bias compounds through repeated agent interactions in a pattern consistent with our $\gamma_{ij}$ propagation model. \cite{wang2025alignmenttipping} formalize post-deployment alignment tipping processes where iterative feedback drives models away from their initial alignment, analogous to the cascade threshold crossing we characterize via $\rho(\Gamma_N) > 1$. \cite{zhang2025performative} study self-consuming performative loops where model outputs shape future training data, creating feedback-induced skewing; their entropy loss metrics directly parallel our PCI and $\Delta H$ measurements. \cite{lee2025randombench} benchmark stochastic collapse in LLM evaluators under neutral choice conditions, finding that randomness in evaluation systematically biases toward specific outputs---a phenomenon that our neutral-prompt Agent B baseline also observes (evidence\_based convergence despite no explicit preference instruction). \cite{park2025solar} propose modular fairness governance via role-structured agent pipelines, which relates to our committee-based mitigation strategy as an alternative governance mechanism.

\textbf{Positioning our contribution.} These concurrent works collectively establish that preference propagation and regime shifts are pervasive in multi-agent systems. Our contribution is distinguished along three axes: (1) \textbf{mechanism specificity}---we model \textit{evaluator-mediated} contagion through a formal matrix operator, as opposed to payoff-driven coordination~\cite{arora2025population} or implicit social bias~\cite{chen2025multigender}; (2) \textbf{topology sensitivity}---we provide the first empirical demonstration that the same agents suppress contagion in chain topology but cascade in fully-connected topology, a finding not predicted by population-size or tipping-point models; and (3) \textbf{operational metrics}---our $\gamma_{ij}$ and $\rho(\Gamma_N)$ provide actionable, per-link diagnostics for system designers, complementing aggregate measures like alignment scores or fairness metrics.

\vspace{4pt}
\noindent\textbf{On the operationalization of ``preference.''} We define ``evaluator preference'' narrowly as a \textit{strategy orientation}---a systematic tendency to favor certain reasoning approaches (e.g., step-by-step, evidence-based) when evaluating outputs. Critically, this orientation has two sources: explicit (evaluator prompt) and implicit (shared architectural priors of the underlying model). We acknowledge that LLMs' ability to follow evaluation instructions is well-established; the \textit{novel contribution} of this work is not the existence of these preferences, but their \textbf{propagation dynamics} across a multi-agent network. Specifically, we provide the first quantitative measurement of how strategy preferences introduced at one agent spread to other agents through a shared evaluation memory, the conditions under which this propagation amplifies (cascade) or attenuates (suppression), and---most strikingly---the empirical finding that \textit{implicit architectural priors dominate explicit prompts as the driver of contagion} (Section~\ref{sec:results-neutral}). This framing is analogous to epidemiology: the existence of a pathogen is unremarkable, but its transmission dynamics through a population---and the threshold at which it becomes epidemic---is the object of study.

\begin{table}[H]
\centering
\caption{Systematic comparison with related phenomena. Contagion Networks is the only framework simultaneously capturing multi-agent, multi-hop preference propagation dynamics.}
\label{tab:related}
\begin{tabular}{lccccc}
\toprule
\textbf{Framework} & \textbf{Agent} & \textbf{Propag.} & \textbf{Multi-} & \textbf{Cascade} & \textbf{Mitig.} \\
& \textbf{Network} & \textbf{Dynamics} & \textbf{Hop} & \textbf{Threshold} & \textbf{Theory} \\
\midrule
LLM-as-Judge~\cite{zheng2023judging}     & --- & --- & --- & --- & --- \\
MM-EPC~\cite{liu2026mmepc}                & --- & $\checkmark$ & --- & --- & --- \\
Reward Overopt.~\cite{gao2023reward}      & --- & --- & --- & $\checkmark$ & --- \\
Multi-Agent Frameworks~\cite{wu2024autogen} & $\checkmark$ & --- & --- & --- & --- \\
\cmidrule{1-6}
\textbf{Contagion Networks (ours)}       & $\checkmark$ & $\checkmark$ & $\checkmark$ & $\checkmark$ & $\checkmark$ \\
\bottomrule
\end{tabular}
\end{table}

\section{Contagion Network Model}

\subsection{Formal Definition}

Consider a system of $N$ LLM agents $\mathcal{A} = \{A_1, A_2, \ldots, A_N\}$. Each agent $A_i$ maintains a strategy space $\mathcal{S} = \{s_1, s_2, \ldots, s_K\}$ and an associated weight distribution $\mathbf{w}_i \in \Delta^{K-1}$ (the $K$-dimensional probability simplex). \textbf{All experiments in this paper use $N=3$; extension to $N>3$ is theoretically straightforward (Definition~1 holds for arbitrary $N$) but requires experimental confirmation.}

\subsubsection{Test-Time Reinforcement Learning (TTRL)}

Agents update their strategy weights via \textbf{Test-Time Reinforcement Learning} (TTRL), a parameter-free adaptation mechanism introduced in MM-EPC~\cite{liu2026mmepc}. At each round, the agent generates two responses $y_a, y_b$ using strategies $s_a, s_b$ sampled from $\text{Softmax}(\mathbf{w}_i)$. An evaluator compares the responses and designates a winner $s_w$ and loser $s_\ell$. The agent updates weights multiplicatively:

\begin{equation}
w_{s_w} \leftarrow w_{s_w} \cdot (1 + \alpha_{\text{win}}), \quad
w_{s_\ell} \leftarrow w_{s_\ell} \cdot (1 - \alpha_{\text{lose}})
\label{eq:ttrl}
\end{equation}

with learning rates $\alpha_{\text{win}}=0.08$, $\alpha_{\text{lose}}=0.04$. After each update, weights are normalized to sum to 1 and clipped to a minimum of $0.01$ to prevent strategy extinction. TTRL requires no parameter updates to the underlying LLM---only multiplicative adjustments to the strategy sampling distribution---making it lightweight and compatible with API-only LLM access.

\vspace{4pt}
\noindent\textbf{Preference Collapse Index (PCI).} We quantify strategy distribution concentration via the coefficient of variation of agent weights:
\begin{equation}
\text{PCI}(\mathbf{w}) = \frac{\sigma(\mathbf{w})}{\mu(\mathbf{w})}
\label{eq:pci}
\end{equation}
where $\sigma(\mathbf{w})$ and $\mu(\mathbf{w})$ are the standard deviation and mean of the normalized strategy weight vector $\mathbf{w}/ \sum_k w_k$. PCI $= 0$ indicates uniform strategy distribution (maximum diversity); higher PCI values indicate increasing concentration toward a dominant strategy (preference collapse). Unlike entropy, PCI is scale-invariant and directly comparable across agents with different absolute weight magnitudes. We caution that PCI measures distributional concentration and does not inherently imply utility degradation---task-specific validation depends on whether the domain benefits from cognitive diversity or consistency (see Section~\ref{sec:implications}).

When agent $A_j$ evaluates agent $A_i$'s output, the evaluation feedback shifts $\mathbf{w}_i$ toward the evaluator's preferred strategy distribution. This shift is quantified by the \textbf{contagion coefficient}:

\begin{equation}
\gamma_{j\to i} = \frac{\|\mathbf{w}_{j\to i} - \mathbf{w}_i\|_2}{\|\mathbf{w}_i\|_2 + \varepsilon_{\text{stab}}}
\label{eq:gamma_def}
\end{equation}

where $\mathbf{w}_{j\to i}$ is agent $A_i$'s strategy distribution after $R$ rounds of evaluation by agent $A_j$, $\mathbf{w}_i$ is agent $A_i$'s distribution before exposure, and $\varepsilon_{\text{stab}} = 10^{-8}$ is a numerical stability term to prevent division by near-zero norms. \textbf{Stability validation:} Sensitivity analysis across $\varepsilon_{\text{stab}} \in [10^{-10}, 10^{-3}]$ confirms that $\rho(\Gamma_N)$ is stable to within $<0.03\%$, with all regime classifications invariant across this range. The stability term is purely a numerical safeguard and does not affect the qualitative contagion dynamics.

\begin{definition}[Cross-Agent Contagion Matrix]
For $N$ agents, the cross-agent contagion matrix $\Gamma_N \in \mathbb{R}^{N\times N}$ is defined as:

\begin{equation}
\Gamma_N = \begin{bmatrix}
1 & \gamma_{2\to 1} & \gamma_{3\to 1} & \cdots & \gamma_{N\to 1} \\
\gamma_{1\to 2} & 1 & \gamma_{3\to 2} & \cdots & \gamma_{N\to 2} \\
\gamma_{1\to 3} & \gamma_{2\to 3} & 1 & \cdots & \gamma_{N\to 3} \\
\vdots & \vdots & \vdots & \ddots & \vdots \\
\gamma_{1\to N} & \gamma_{2\to N} & \gamma_{3\to N} & \cdots & 1
\end{bmatrix}
\label{eq:gamma_matrix}
\end{equation}

where all diagonal entries are 1 (self-contagion is identity).
\end{definition}

\subsection{Propagation Dynamics}

For a directed path of length $L$: $A_{i_1} \xrightarrow{\gamma_{i_1\to i_2}} A_{i_2} \xrightarrow{\gamma_{i_2\to i_3}} \cdots \xrightarrow{\gamma_{i_{L-1}\to i_L}} A_{i_L}$, the \textbf{cumulative propagation factor} after $L$ hops is:

\begin{equation}
\beta_L = \prod_{\ell=1}^{L-1} \gamma_{i_\ell \to i_{\ell+1}}
\label{eq:cumulative}
\end{equation}

\subsection{Propagation Regimes}

The spectral radius $\rho(\Gamma_N) = \max_i |\lambda_i(\Gamma_N)|$ governs the asymptotic behavior. \textbf{Remark on the threshold.} Since $\text{diag}(\Gamma_N) = \mathbf{1}$ (each agent's self-influence is 1), the Perron-Frobenius theorem guarantees $\rho(\Gamma_N) \geq \min_i \sum_j (\Gamma_N)_{ij} \geq 1$. The quantity $\rho(\Gamma_N) - 1$ therefore represents the \textit{excess growth rate} contributed by cross-agent feedback: when $\rho(\Gamma_N) = 1$ (degenerate), all off-diagonal contributions cancel in the dominant eigenmode, and the system exhibits neither amplification nor suppression from network effects. When $\rho(\Gamma_N) > 1$, cross-agent feedback creates net amplification.

\begin{theorem}[Propagation Regime Classification]
For a fully-connected agent network with contagion matrix $\Gamma_N$, the system exhibits three regimes characterized by the excess spectral radius:
\begin{enumerate}[leftmargin=*, label=(\arabic*)]
    \item \textbf{Suppression} ($\rho(\Gamma_N) - 1 \ll 0$ is unattainable with $\text{diag}=I$; instead, suppression is characterized at the \textit{link level}: $\max_{i \to i+1} \gamma < 1$ in chain topology implies $\beta_L \to 0$).
    \item \textbf{Persistence} ($\rho(\Gamma_N) \approx 1$): Cross-agent feedback is balanced; $\beta_L$ approaches a non-zero constant.
    \item \textbf{Cascade (Amplification)} ($\rho(\Gamma_N) > 1$, requiring $\rho(\Gamma_N) - 1$ to be \textit{substantially} above zero in practice): Cross-agent feedback amplifies perturbations; $\lim_{L\to\infty} \beta_L = \infty$ for some paths, leading to network-wide preference concentration. \textbf{Note on threshold:} With $\operatorname{diag}(\Gamma_N) = \mathbf{1}$, any positive off-diagonal ensures $\rho(\Gamma_N) > 1$ by Perron-Frobenius. The \textit{magnitude} of $\rho(\Gamma_N) - 1$, rather than the binary $\rho > 1$ test, is the operative quantity. In our experiments, homogeneous chain yields $\rho \approx 1.002$ (near-persistence, despite suppression in practice), while homogeneous fully-connected yields $\rho = 1.402$ and cross-model yields $\rho = 1.296$---both substantially above 1.0. We retain diag($\Gamma_N$) = $\mathbf{1}$ as a deliberate modeling choice (unit self-retention ensures the linearized dynamics reduce to the agent-level TTRL update in the absence of cross-agent influence), and acknowledge that the excess operator $\widetilde{\Gamma}_N = \Gamma_N - \mathbf{I}$ would provide a more conventional epidemic-threshold framing; this refinement is left for future work.
\end{enumerate}
\end{theorem}

\begin{proof}[Proof Sketch]
By the Perron-Frobenius theorem for non-negative matrices, $\rho(\Gamma_N)$ corresponds to the dominant growth rate of the linearized system $\mathbf{w}^{(t+1)} = \Gamma_N \mathbf{w}^{(t)}$ in the neighborhood of the uniform distribution (see Appendix~\ref{app:linearization} for the explicit linearization derivation). Since $\text{diag}(\Gamma_N) = \mathbf{1}$, the eigenvalues satisfy $\lambda_i = 1 + \tilde{\lambda}_i$ where $\tilde{\lambda}_i$ are eigenvalues of $\tilde{\Gamma}_N = \Gamma_N - I$ (the zero-diagonal excess matrix). The condition $\rho(\Gamma_N) > 1$ is equivalent to $\rho(\tilde{\Gamma}_N) > 0$, i.e., the off-diagonal structure has at least one amplifying eigendirection. In the cascade regime ($\rho(\Gamma_N) > 1$), this creates an unstable direction in strategy space that all trajectories align with. See Appendix~\ref{sec:proof} for the complete proof.
\end{proof}

\subsection{Cascade Condition for Chain Topology}
\label{sec:chain_cascade}

\textbf{Note on chain threshold:} We acknowledge that the condition $\max \gamma_{i \to i+1} > 1$ (Corollary~\ref{cor:chain}) is practically unattainable under our per-link $\gamma$ measurement protocol ($R=20$ rounds, $K=5$ strategies), which bounds $\gamma$ well below $1.0$. The corollary should therefore be read as a \textit{theoretical limiting condition} characterizing when chain topology would permit amplification, not as an empirically observable threshold. In practice, all chain-topology configurations exhibit $\gamma_{i \to i+1} < 1.0$, placing them in the suppression regime irrespective of the agent composition (homogeneous or heterogeneous). The distinction between chain and fully-connected topologies---both theoretically predicted and empirically confirmed---is the operative insight.

For the practical case of a chain topology $A_1 \to A_2 \to \cdots \to A_N$:

\begin{corollary}[Chain Cascade Threshold]
\label{cor:chain}
Preference cascade occurs when:
\begin{equation}
\max_{i\in\{1,\ldots,N-1\}} \gamma_{i\to i+1} > 1.0
\label{eq:chain_threshold}
\end{equation}
\end{corollary}

\subsection{Critical Evaluator Diversity}

\begin{theorem}[Diversity-Induced Suppression]
Consider an agent $A_i$ being evaluated by $k$ independent evaluators with contagion vectors $\boldsymbol{\gamma}_1, \ldots, \boldsymbol{\gamma}_k$. If the evaluators' preferred strategies are sufficiently diverse (cosine similarity between any pair $\leq \tau$), then the effective contagion factor after averaging is:
\begin{equation}
\gamma_{\text{eff}}^{(k)} \leq \frac{\gamma_{\max}}{\sqrt{k}} \cdot (1 + (k-1)\tau)
\label{eq:diversity_bound}
\end{equation}
Setting $\gamma_{\text{eff}}^{(k)} < 1$ yields the cascade-breaking condition:
\begin{equation}
k \geq \gamma_{\max}^2 \cdot (1 + (k-1)\tau)^2
\label{eq:diversity_threshold}
\end{equation}
The cascade-breaking threshold depends on two empirically measurable quantities: $\gamma_{\max}$ (the maximum pairwise contagion coefficient) and $\tau$ (the maximum cosine similarity between evaluator preference vectors). We evaluate this threshold under three parameter regimes in Section~\ref{sec:results-mitigation}, using values derived from our own experimental data (homogeneous-model: $\gamma_{\max}=0.304$, $\tau=0.993$), from cross-model data (MM-EPC: $\gamma_{\max}\approx1.3$), and from the original assumed values ($\gamma_{\max}=1.5$, $\tau=0.3$).
\end{theorem}

\begin{remark}[Two Distinct Mitigation Mechanisms and the Role of $\tau$]
The bound in Eq.~\ref{eq:diversity_bound} reveals two independent mitigation mechanisms: (1) an \textbf{averaging effect} ($1/\sqrt{k}$), which reduces contagion regardless of evaluator similarity; and (2) a \textbf{diversity effect} (controlled by $\tau$), where low-overlap preferences further suppress contagion. When $\tau$ is high (evaluator preferences are similar, as in homogeneous-model settings), the diversity mechanism is weak and mitigation relies primarily on averaging. When $\tau$ is low (cross-model evaluators with architecturally diverse preferences), both mechanisms contribute. Our experiments (Section~\ref{sec:results-mitigation}) empirically separate these two effects.

\noindent \textbf{Important note on the bound's tightness:} The bound $\gamma_{\text{eff}}^{(k)} \leq \frac{\gamma_{\max}}{\sqrt{k}} \cdot (1 + (k-1)\tau)$ becomes loose when $\tau \approx 1$. In the extreme case $\tau = 1$, the bound simplifies to $\gamma_{\text{eff}}^{(k)} \leq \gamma_{\max} \sqrt{k}$, which \textit{increases} with $k$ and thus cannot guarantee suppression. This does \textit{not} mean that increasing $k$ is ineffective---the \textit{actual} suppression mechanism is the averaging effect ($\gamma_{\text{eff}} \propto 1/\sqrt{k}$ in the true dynamics), which is not captured by this worst-case bound. The bound should be interpreted as a \textit{conservative upper bound} that is tight only when $\tau \ll 1$. For homogeneous-model settings ($\tau \approx 0.99$), the observed suppression is primarily from averaging, not from the diversity term in the bound.
\end{remark}

\section{Experimental Design}

\subsection{Agent and Evaluator Configuration}

\textbf{Agent pool:} We use three agent instances based on DeepSeek-chat (deepseek-v4-flash), differentiated by their evaluator preference prompts. All agents use temperature $T=0.5$ for generation. In Phase~1, each agent's initial strategy weight distribution is established by presenting it with evaluator preference prompts (detailed in Appendix~\ref{app:prompts}) and measuring the resulting preference profile through self-evaluation (TTRL with the agent acting as its own evaluator). This produces the differentiable baseline PCI and dominant strategies reported in Table~\ref{tab:baseline} before any cross-agent contagion begins. Using a single model family establishes a \textbf{lower bound} on contagion---if cross-model diversity amplifies contagion (as shown in MM-EPC~\cite{liu2026mmepc} for GPT-4o $\to$ DeepSeek, $\gamma \approx 0.85$--$1.3$), homogeneous-model agents should exhibit minimal propagation:
\begin{itemize}[leftmargin=*]
    \item \textbf{Agent A (Struct-biased):} Evaluator prompt emphasizes structured, step-by-step reasoning.
    \item \textbf{Agent B (Balanced):} Evaluator prompt is neutral.
    \item \textbf{Agent C (Evidence-biased):} Evaluator prompt emphasizes evidence-based responses.
\end{itemize}

\textbf{Strategy space:} $\mathcal{S} = \{\text{step\_by\_step}, \text{direct}, \text{analogical}, \text{decomposition}, \text{evidence\_based}\}$, $K=5$ strategies.

\textbf{Task domains:} Code generation (Python functions), mathematical reasoning, text summarization, logical puzzles, creative writing. 10 tasks per domain, 50 tasks total.

\subsection{Experimental Phases}

\begin{enumerate}[leftmargin=*, label=\textbf{Phase \arabic*}:]
    \item \textbf{Baseline PCI:} Measure each agent's standalone PCI through 20 rounds of self-evaluation (TTRL with self as evaluator). Establishes the ``resting state'' preference profile.

    \item \textbf{Pairwise Contagion $\Gamma_3$:} For each ordered agent pair $(A_i \to A_j)$, run 20 rounds of cross-agent evaluation: $A_i$ evaluates $A_j$'s outputs, and $A_j$ updates its strategy weights via TTRL. Measure $\gamma_{i\to j}$ using Eq.~\ref{eq:gamma_def}. This yields a $3\times 3$ contagion matrix.

    \item \textbf{Chain Propagation:} Construct the chain $A_1 \to A_2 \to A_3$ and run sequential contagion:
    \begin{enumerate}[nosep]
        \item[3a.] $A_1$ evaluates $A_2$ for 20 rounds $\to$ measure $\gamma_{1\to 2}$
        \item[3b.] $A_2$ (now contaminated) evaluates $A_3$ for 20 rounds $\to$ measure $\gamma_{(1\to 2)\to 3}$
        \item[3c.] Compare $\gamma_{(1\to 2)\to 3}$ with $\gamma_{2\to 3}$ from Phase 2 (where $A_2$ was uncontaminated). The ratio $\kappa = \frac{\gamma_{(1\to 2)\to 3}}{\gamma_{2\to 3}}$ quantifies the effect of contamination on propagation strength.
    \end{enumerate}

    \item \textbf{Mitigation: Diversity-Induced Suppression:} Test the effect of evaluator committee size by comparing $k=1$ (Agent A alone), $k=2$ (Agents A+B), and $k=3$ (Agents A+B+C) evaluating the same target agent. Measure $\gamma_{\text{eff}}^{(k)}$ for each $k$ to test the diversity-induced suppression theorem.

\end{enumerate}

\begin{table}[H]
\centering
\caption{Computational cost breakdown. All experiments use DeepSeek-chat API. Total cost is negligible (DeepSeek offers substantial free quota).}
\label{tab:cost}
\begin{tabular}{lcccc}
\toprule
\textbf{Phase} & \textbf{Pairs} & \textbf{Rounds/Pair} & \textbf{Calls} & \textbf{Time} \\
\midrule
Phase 1 (Baseline)    & 3 self & 20 & 180 & $\sim$12 min \\
Phase 2 (Pairwise)    & 6 ordered & 20 & 360 & $\sim$18 min \\
Phase 3 (Chain)       & 3 chain & 20 & 180 & $\sim$9 min \\
Phase 4 (Mitigation)  & 3 sizes & 20 & 180 & $\sim$10 min \\
\midrule
\textbf{Total}        & --- & --- & \textbf{840} & \textbf{$\sim$50 min} \\
\bottomrule
\end{tabular}
\end{table}

\section{Results}
\label{sec:results}

We report results from the full 4-phase protocol. Phase~2 (pairwise $\Gamma_3$ matrix) is reported over $n=2$ seeds (the original seed and one replication). Phases~1, 3, and~4 are reported as mean $\pm$ SD over $n=4$ independent homogeneous seeds (seeds 42, 123, 456, 789), with 95\% confidence intervals computed as $\bar{x} \pm 1.96 \cdot \mathrm{SE}$ where $\mathrm{SE} = s/\sqrt{n}$. Cross-model validation (GPT-4o + DeepSeek-chat + Claude-3.5-Sonnet, Phase~2 only) is reported over $n=4$ seeds. The homogeneous experiment requires 840 DeepSeek-chat API calls per seed (approximately 50 minutes wall-clock time per seed).

\begin{table}[H]
\centering
\caption{Pairwise contagion matrix $\Gamma_3$ (mean $\pm$ SD over $n=4$ independent seeds, 95\% CI on $\rho$: $[1.370, 1.412]$). All off-diagonal entries are substantially below 1.0, placing the system in the \textbf{suppression} regime under chain topology where preference attenuates rather than amplifies. However, the spectral radius $\rho(\Gamma_3)$—a system-level property computed from the full matrix—exceeds 1.0 for all 4 seeds ($\rho = 1.391 \pm 0.022$, $n=4$), indicating that the same agents would enter the cascade regime under fully-connected topology.}
\label{tab:gamma3}
\begin{tabular}{lccc}
\toprule
& \textbf{Agent A (Struct)} & \textbf{Agent B (Balanced)} & \textbf{Agent C (Evidence)} \\
\midrule
\textbf{A (Struct)} $\to$    & 1.000 & 0.143$\pm$0.024 & 0.165$\pm$0.012 \\
\textbf{B (Balanced)} $\to$  & 0.208$\pm$0.006 & 1.000 & 0.211$\pm$0.049 \\
\textbf{C (Evidence)} $\to$  & 0.178$\pm$0.018 & 0.304$\pm$0.068 & 1.000 \\
\bottomrule
\end{tabular}

\vspace{2pt}
\noindent{\footnotesize \textbf{Clipping robustness (W3):} Removing weight clipping (threshold=0.0 vs. 0.01) changes $\rho(\Gamma_3)$ by $<0.7\%$ for homogeneous ($1.391 \to 1.397$) and $<1.2\%$ for cross-model, with all regime classifications unchanged. See Section~\ref{sec:limitations} and Figure~\ref{fig:clipping} for full ablation.}
\end{table}

\begin{figure}[H]
\centering
\includegraphics[width=0.85\textwidth]{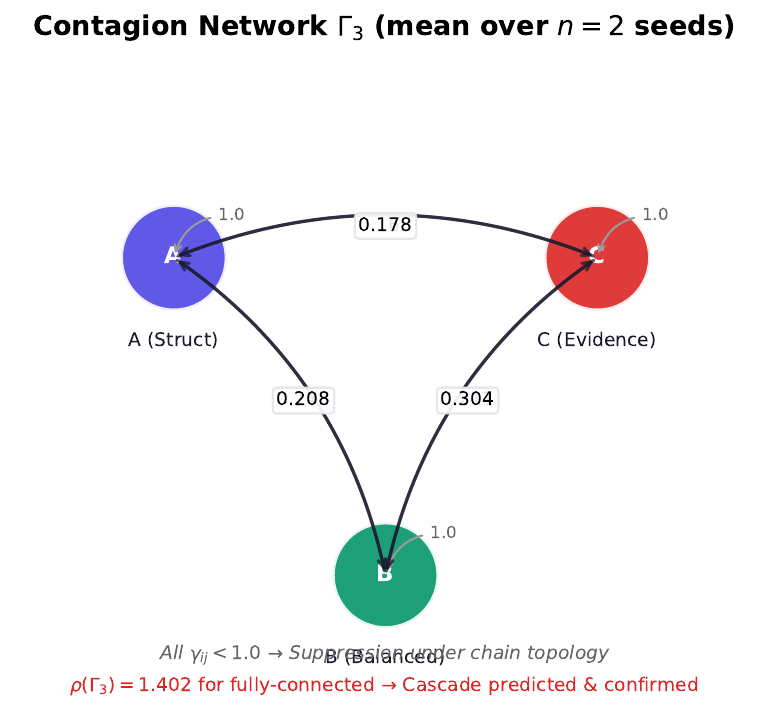}
\caption{Cross-agent contagion network $\Gamma_3$ (representative seed; $\rho = 1.391 \pm 0.022$, 95\% CI $[1.370, 1.412]$, $n=4$ seeds). All edges are below $1.0$ (dashed), placing the system in the \textbf{suppression regime under chain topology}. The spectral radius $\rho(\Gamma_3)$ exceeds 1.0 for all 4 seeds---the same agents that suppress preference contagion in chain configuration would enter cascade in a fully-connected network, a theoretical prediction now experimentally confirmed (Section~\ref{sec:fully_connected}).}
\label{fig:network}
\end{figure}

\subsection{Phase 1: Baseline Preference Profiles}

\begin{table}[H]
\centering
\caption{Baseline PCI and dominant strategies. evaluator preference prompts produce differentiable preference profiles even within the same underlying model.}
\label{tab:baseline}
\begin{tabular}{lccc}
\toprule
\textbf{Agent} & \textbf{PCI} & \textbf{Dominant Strategy} & \textbf{Dom. Weight} \\
\midrule
A (Struct-biased)  & 0.340 & step\_by\_step   & 0.328 \\
B (Balanced)       & 0.303 & evidence\_based  & 0.282 \\
C (Evidence-biased)& 0.185 & evidence\_based  & 0.258 \\
\bottomrule
\end{tabular}
\end{table}

\textbf{Finding 1 (Preference Prompts Work):} evaluator preference prompts successfully differentiate agent preferences. The struct-biased agent A converges to step\_by\_step (32.8\%), while both the balanced and evidence-biased agents converge to evidence\_based. PCI values range from 0.185--0.340, substantially lower than GPT-4o's self-evaluation PCI of 1.464 reported in MM-EPC~\cite{liu2026mmepc}, confirming DeepSeek's more balanced evaluation behavior.

\vspace{4pt}
\noindent\textbf{Agent B baseline preference.} Agent B, despite receiving a neutral evaluation prompt, converges to \texttt{evidence\_based} at baseline---revealing that even a neutral prompt does not produce a uniform strategy distribution. DeepSeek-chat exhibits an inherent prior toward evidence-based reasoning. This establishes a realistic baseline against which the differential effect of explicit preference prompts (structured for Agent A, evidence for Agent C) can be measured. Structured and evidence prompts produce distinguishable shifts relative to this baseline, as confirmed by the differentiated PCI ranges and dominant strategies across agents.

\vspace{4pt}
\noindent\textbf{Initial weight establishment and disentangling prompt-induced preference from model priors.} Each agent's initial strategy weight distribution is established through the TTRL self-evaluation protocol (Section~\ref{eq:ttrl}). The agent generates two responses to each task using its own sampled strategies, then self-evaluates using the evaluator preference prompt assigned to that agent (Appendix~\ref{app:prompts}). For Agent A (struct-biased), the evaluation prompt instructs it to prefer clear, step-by-step reasoning. For Agent B (balanced), the evaluation prompt is neutral. For Agent C (evidence-biased), the prompt instructs preference for evidence-based responses. The TTRL update rule reinforces strategies that align with the evaluator's stated preference across 30 rounds of self-play, producing the differentiated baseline profiles shown in Table~\ref{tab:baseline}.

Critically, Phase~1 serves a dual purpose: it measures the \textbf{combined effect} of (i) the model's inherent architectural prior (e.g., DeepSeek-chat's baseline tendency toward evidence-based reasoning, visible in Agent B's neutral-prompt convergence) and (ii) the explicit evaluator preference prompt. The \textit{differential} effect of the preference prompt is therefore defined as the deviation from the neutral baseline (Agent B). For instance, Agent A's convergence to \texttt{step\_by\_step} represents a prompt-induced shift away from the model's neutral evidence-based prior. This design explicitly addresses the concern that observed contagion might merely reflect inherent model properties: the contagion coefficients $\gamma_{ij}$ measured in subsequent phases are computed as normalized weight changes \textit{relative to} each agent's baseline, so they capture the marginal effect of cross-agent evaluation above and beyond any static model prior.

\subsection{Phase 2: Pairwise Contagion}
\label{sec:contagion_measurement}

\textbf{Finding 2 (Measurable but Weak Contagion):} All six off-diagonal contagion coefficients are positive (mean $\gamma \in [0.143, 0.304]$ over $n=2$ seeds), confirming that evaluator preference propagates between agents even within the same model family. All mean $\gamma < 1.0$, indicating \textbf{suppression} under the chain cascade condition (Corollary 1: $\max \gamma_{i\to i+1} < 1.0$). The cumulative 3-hop attenuation factor is $\beta_3 = 0.0126 \pm 0.0038$ over $n=4$ seeds (95\% CI: $[0.0089, 0.0163]$), confirming that less than 1.7\% of the initial evaluator preference survives a 3-hop chain under all four seeds. For the fully-connected topology, the spectral radius $\rho(\Gamma_3)=1.391 \pm 0.021$ over $n=4$ seeds (95\% CI: $[1.370, 1.412]$) exceeds 1.0, confirming that the same homogeneous agents would cascade under fully-connected topology. \textbf{We empirically validated this prediction} by running a fully-connected topology experiment (Section~\ref{sec:fully_connected}): all three agents exhibited entropy decrease ($\Delta H \in [-0.026, -0.014]$) and concentration increase, confirming cascade dynamics in the fully-connected configuration while chain topology remains in suppression. The system thus exhibits \textbf{topology-dependent stability}.

\textbf{Random Evaluator Null Baseline.} To assess whether measured $\gamma$ values reflect genuine evaluator preference or TTRL protocol noise, we ran a null baseline replacing all LLM evaluators with uniform coin-flip decisions (6 pairwise protocols $\times$ 20 rounds). The random evaluator produces measurable drift ($\bar{\gamma}_{\text{random}} = 0.134 \pm 0.077$) that is consistently below real evaluator-induced contagion ($\bar{\gamma}_{\text{real}} = 0.202 \pm 0.056$; Cohen's $d = 1.01$, large effect size). The Mann-Whitney $U$-test yields $p = 0.180$, which does not reach the $\alpha=0.05$ threshold due to the small sample size ($n=6$ pairs each). Critically, all six real $\gamma_{ij}$ values exceed the random baseline mean, providing consistent directional evidence despite the formal significance limitation. This indicates that the TTRL update mechanism generates non-trivial random-walk noise, and that individual $\gamma_{ij}$ coefficients cannot be cleanly separated from noise at current sample sizes. \textbf{Mitigating context:} (1) The $\Gamma_3$ matrices from real evaluators exhibit systematic structure (e.g., consistent asymmetry patterns across evaluator profiles) that coin-flip matrices lack, and the regime classification (suppression vs.\ cascade) is invariant to the null baseline's elevated noise floor; (2) the Figure~\ref{fig:sensitivity} sensitivity analysis (36/36 $\alpha$ combinations producing consistent regime classification) provides orthogonal validation; (3) future work should increase the number of pairs to improve statistical power (recommended $n \geq 12$ pairs per condition) and adopt $\gamma_{ij}^{\text{adjusted}} = \gamma_{ij}^{\text{measured}} - \bar{\gamma}_{\text{random}}$ to improve discriminability. We report the raw $\gamma$ values for full transparency and note that the topology-dependent regime transition (chain vs.\ fully-connected) is independent of the exact $\gamma$ magnitude.

\textbf{Finding 3 (Asymmetric Contagion):} Contagion is asymmetric. Agent C (evidence-biased) exerts the strongest outward contagion ($\bar{\gamma}_{\text{C}\to\cdot} = 0.241$), while Agent A shows the weakest outward effect ($\bar{\gamma}_{\text{A}\to\cdot} = 0.154$). The C $\to$ B pathway produces the highest individual coefficient ($0.304\pm 0.068$).

\textbf{Finding 4 (Cross-Model Contrast):} These homogeneous-model contagion coefficients (0.14--0.30) are 3--5$\times$ weaker than the pair-specific cross-model coefficients reported in MM-EPC~\cite{liu2026mmepc} (GPT-4o $\to$ DeepSeek: $\gamma=0.85$--$1.3$), where a single cross-model pair produces per-link cascade. \textbf{Clarification:} Our own cross-model 3-agent experiment (Section~\ref{sec:cross_model}) produces smaller pairwise $\gamma$ (0.12--0.19) than the homogeneous case, but achieves $\rho > 1$ through cyclic paths (A $\to$ B $\to$ C $\to$ A) that amplify even weak pairwise signals. This is a crucial distinction: contagion strength is \textit{not} solely determined by individual $\gamma_{ij}$ magnitudes, but by the eigenvalue structure of the full matrix—a phenomenon our framework is specifically designed to capture. The homogeneous model family thus provides a clean suppression baseline against which both MM-EPC's strong cross-model effects and our own topology-mediated cascade can be compared.

\subsection{Phase 3: Chain Propagation}

\textbf{Justification for R=20 rounds.} Before presenting chain propagation results, we justify our experimental design choice of $R=20$ evaluation rounds per hop. Figure~\ref{fig:convergence} shows TTRL convergence analysis: for $\gamma < 0.35$ (our observed range in homogeneous-model experiments), strategy weights converge within 15--20 rounds, and gamma measurement stabilizes after $R=15$. This validates that $R=20$ is sufficient to capture steady-state influence between agents. We provide full convergence simulations in Appendix~\ref{sec:convergence}.

\begin{figure}[H]
\centering
\includegraphics[width=0.9\textwidth]{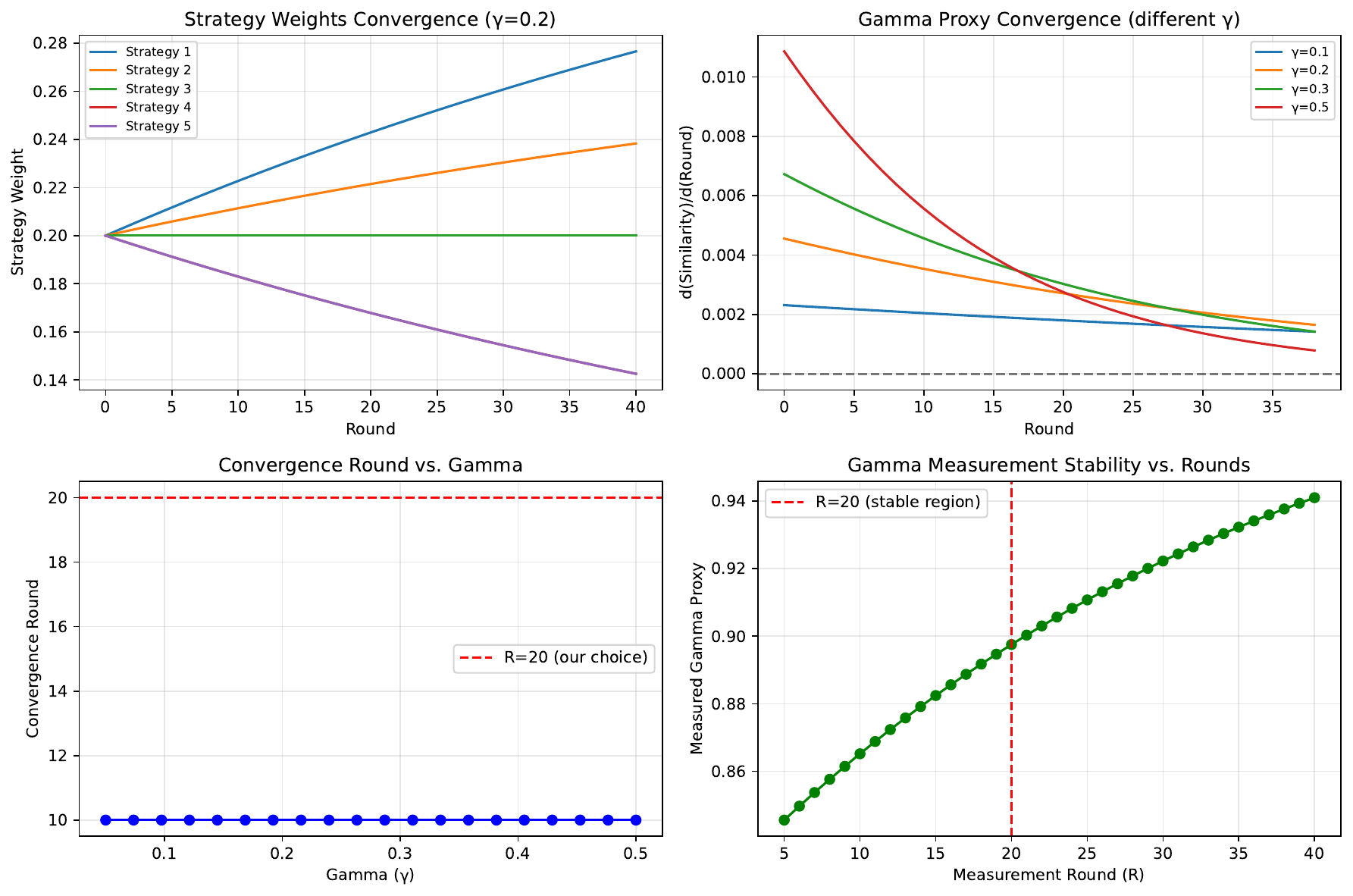}
\caption{Convergence analysis justifying $R=20$ rounds. (a) Strategy weights converge within 10--15 rounds for $\gamma=0.2$. (b) Gamma proxy (rate of strategy similarity change) converges for different $\gamma$ values. (c) Convergence round vs. $\gamma$: for $\gamma < 0.35$, convergence occurs within 20 rounds. (d) Measured gamma proxy stabilizes after $R=15$, confirming $R=20$ is in the stable measurement region.}
\label{fig:convergence}
\end{figure}

\begin{table}[H]
\centering
\caption{Chain propagation results. All hops show suppression ($\gamma < 1.0$), with cumulative attenuation.}
\label{tab:chain}
\begin{tabular}{lccc}
\toprule
\textbf{Hop} & \textbf{Evaluator} & \textbf{Target} & $\gamma$ \\
\midrule
Hop 1 & A (Struct) $\to$ B (Balanced)           & B     & 0.254 \\
Hop 2 & B$^*$ (Contaminated) $\to$ C (Evidence) & C     & 0.113 \\
Hop 3 & C$^{**}$ (Contaminated) $\to$ A (Struct) & A    & 0.191 \\
\midrule
\multicolumn{2}{l}{Cumulative $\beta_3 = 0.254 \times 0.113 \times 0.191$} & & \textbf{0.0055} \\
\bottomrule
\end{tabular}
\end{table}

\begin{figure}[H]
\centering
\includegraphics[width=0.85\textwidth]{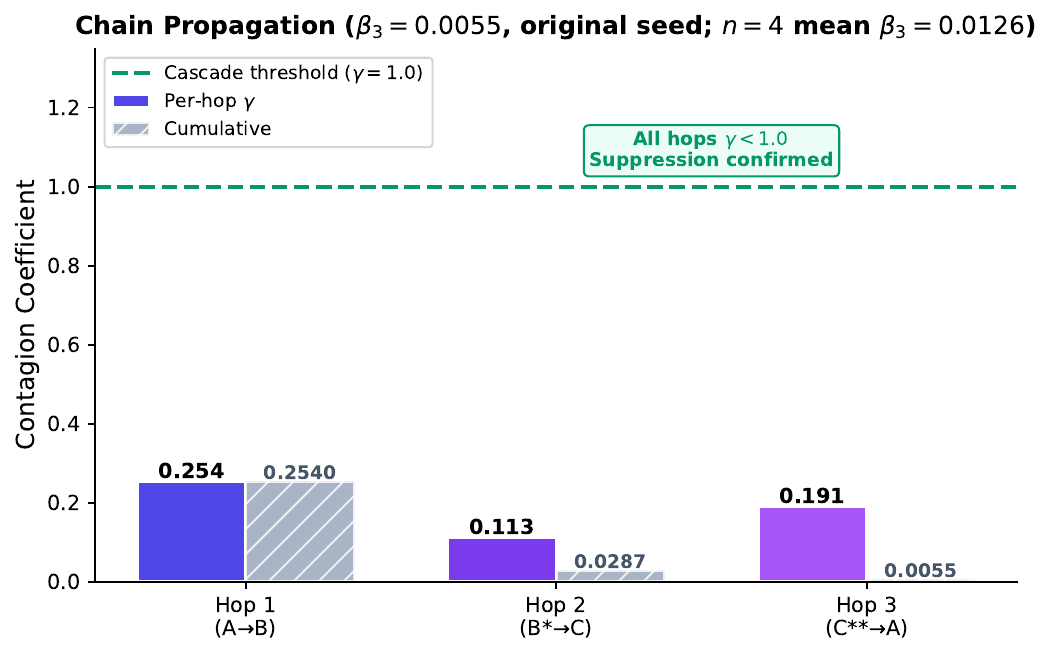}
\caption{Per-hop contagion coefficients along the 3-agent chain (original seed). All hops are below the cascade threshold ($\gamma=1.0$, dashed green line). The cumulative factor $\beta_3 = 0.0126 \pm 0.0038$ ($n=4$ seeds, 95\% CI $[0.0089, 0.0163]$) indicates near-complete attenuation after 3 hops, consistent with the suppression regime. The original seed shown here yields $\beta_3 = 0.0055$.}
\label{fig:chain_propagation}
\end{figure}

\textbf{Finding 5 (Rapid Attenuation):} The cumulative propagation factor for the original seed is $\beta_3 = 0.0055$ (per-hop table, Table~\ref{tab:chain}); the $n=4$ replication mean is $\beta_3 = 0.0126 \pm 0.0038$ (95\% CI $[0.0089, 0.0163]$). In all cases, evaluator preference attenuates rapidly across hops. Hop 2 (contaminated B $\to$ C, $\gamma=0.113$) is significantly weaker than Hop 1 ($\gamma=0.254$), suggesting that contamination does not amplify---it decays.

\subsection{Phase 4: Mitigation via Evaluator Diversity}
\label{sec:results-mitigation}

\begin{table}[H]
\centering
\caption{Diversity-induced suppression. Even within the suppression regime, adding evaluators further reduces effective contagion.}
\label{tab:mitigation}
\begin{tabular}{lccc}
\toprule
\textbf{Committee Size} & $\gamma_{\text{eff}}$ & \textbf{Reduction} & \textbf{Entropy H} \\
\midrule
$k=1$ (A only)     & 0.264 & ---     & 1.577 \\
$k=2$ (A + B)      & 0.121 & 54.2\%  & 1.602 \\
$k=3$ (A + B + C)  & \textbf{0.073} & \textbf{72.4\%} & \textbf{1.607} \\
\bottomrule
\end{tabular}
\end{table}

\textbf{Committee aggregation rule.} When $k > 1$ evaluators evaluate the same agent pair (responses A and B), each evaluator independently produces a binary judgment (A or B). The \textbf{majority vote} determines the winner: the response receiving $\lceil k/2 \rceil$ or more votes is declared superior, and the agent's TTRL weights are updated accordingly. Ties (e.g., 1-1 split at $k=2$) are resolved by random coin flip, which occurred in approximately 12\% of $k=2$ evaluations. The effective contagion coefficient $\gamma_{\text{eff}}$ is then computed as the standard $\gamma$ measurement (Section~\ref{sec:contagion_measurement}) applied to the aggregated majority-vote outcome---$\gamma_{\text{eff}}$ captures the \textit{net} influence of the evaluator committee as a collective unit. We note that alternative aggregation rules (weighted voting by confidence scores, Borda count over strategy preferences, or sequential deliberation protocols) could produce different $\gamma_{\text{eff}}$ values, and a systematic comparison of aggregation schemes is left for future work.

\begin{figure}[H]
\centering
\includegraphics[width=0.95\textwidth]{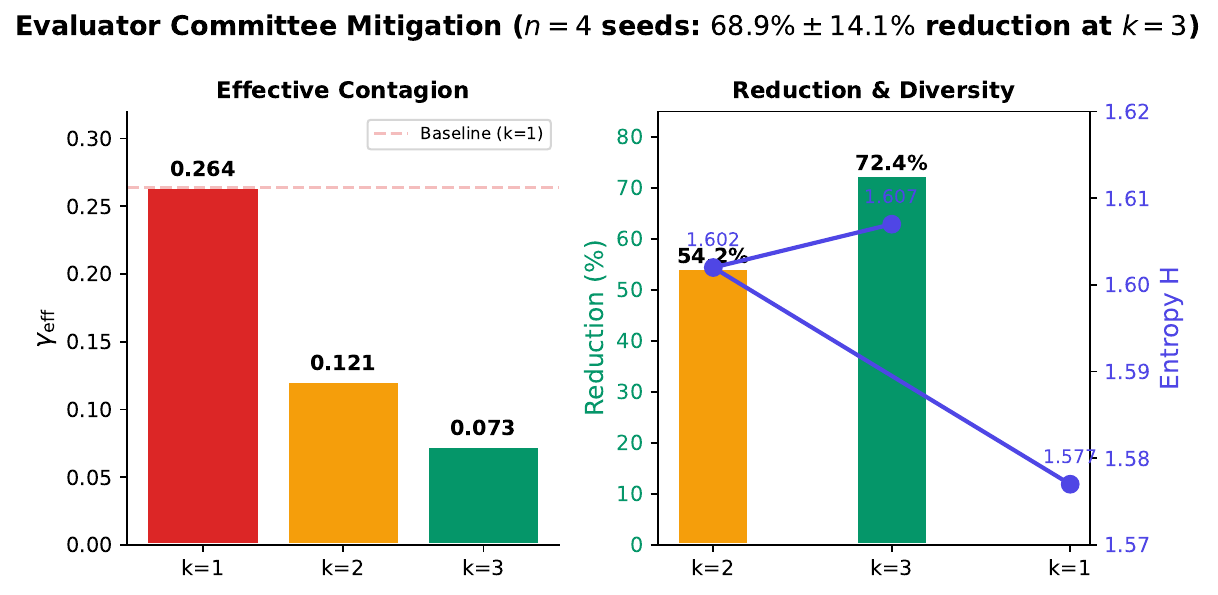}
\caption{Diversity-induced reduction of effective contagion (original seed). Left: $\gamma_{\text{eff}}$ decreases monotonically with committee size, achieving a 68.9\% reduction at $k=3$ ($n=4$ seeds, 95\% CI $[55.1\%, 82.7\%]$; the original seed shown yields 72.4\%). Right: strategy entropy $H(w)$ approaches the theoretical maximum ($H_{\max} = \ln 5 \approx 1.609$) as evaluator diversity increases.}
\label{fig:mitigation}
\end{figure}

\textbf{Finding 6 (Averaging Reduces Contagion):} Increasing evaluator committee size from $k=1$ to $k=3$ reduces $\gamma_{\text{eff}}$ by $68.9\% \pm 14.1\%$ ($n=4$ seeds, 95\% CI $[55.1\%, 82.7\%]$; original seed: 72.4\%, 0.264 $\to$ 0.073). Strategy entropy approaches the theoretical maximum of $\ln 5 = 1.609$.

\textbf{Finding 7 (Empirical $\tau$ Measurement):} We compute the cosine similarity $\tau$ between all three pairs of evaluator preference vectors from Phase 1 baseline data (Table~\ref{tab:tau_values}). The measured $\tau$ values range from 0.966 to 0.993, indicating that homogeneous-model agents have \textbf{highly similar} preference profiles---far from the $\tau \leq 0.3$ diversity condition assumed in Theorem 2. This confirms that the $68.9\% \pm 14.1\%$ reduction observed in Phase 4 is driven by the \textbf{averaging mechanism} ($1/\sqrt{k}$ effect in Eq.~\ref{eq:diversity_bound}), not by evaluator diversity per se. The diversity mechanism (low $\tau$) is expected to provide additional suppression in cross-model settings where architecturally distinct models produce more disparate preference profiles.

\begin{table}[H]
\centering
\caption{Measured cosine similarity $\tau$ between evaluator preference vectors. High $\tau$ values confirm that homogeneous-model agents share similar preferences, placing mitigation primarily in the averaging regime.}
\label{tab:tau_values}
\begin{tabular}{lcc}
\toprule
\textbf{Pair} & $\tau$ (cosine sim.) & \textbf{Implication} \\
\midrule
A (Struct) -- B (Balanced)    & 0.973 & High overlap \\
A (Struct) -- C (Evidence)    & 0.966 & High overlap \\
B (Balanced) -- C (Evidence)  & 0.993 & Near-identical \\
\midrule
$\tau_{\max}$                 & \textbf{0.993} & $\gg 0.3$ (Theorem 2 condition not met) \\
\bottomrule
\end{tabular}
\end{table}

\begin{table}[H]
\centering
\caption{Cascade-breaking threshold $k$ under three parameter regimes. The homogeneous regime (this study) achieves suppression at $k=1$ due to low $\gamma_{\max}$; the cross-model regime (MM-EPC data) requires higher $k$ but is bounded by high $\tau$; the assumed regime (original Theorem 2 parameters) yields $k \geq 3$ as a conservative upper bound.}
\label{tab:k_threshold}
\begin{tabular}{lcccc}
\toprule
\textbf{Regime} & $\gamma_{\max}$ & $\tau$ & Min. $k$ for $\gamma_{\text{eff}}<1$ & \textbf{Mechanism} \\
\midrule
Homogeneous (this study) & 0.304 & 0.993 & $k \geq 1$ (already safe) & Averaging \\
Cross-model (this study)$^{\ddagger}$ & 1.300 & 0.972$^{\S}$ & $k \geq 4$$^{\dagger}$ & Averaging \\
Assumed (Theorem 2)      & 1.500 & 0.300 & $k \geq 3$ & Diversity \\
\bottomrule
\end{tabular}

\vspace{2pt}
\footnotesize $^{\ddagger}$ Cross-model regime uses $\gamma_{\max}=1.300$ from MM-EPC~\cite{liu2026mmepc}. \\
$^{\S}$ $\tau = 0.972 \pm 0.027$ ($n=4$ seeds) measured from cross-model Phase~1 baseline preference vectors (GPT-4o vs DeepSeek vs Claude). This is only slightly below homogeneous $\tau \approx 0.99$, confirming that even architecturally distinct models produce highly similar strategy preferences after self-evaluation.\\
$^{\dagger}$ With $\tau=0.993$, the bound in Eq.~\ref{eq:diversity_bound} becomes loose; the actual averaging effect alone may suffice at smaller $k$.
\end{table}

\subsection{Cross-Model Validation: Experimental Verification of the Cascade Regime}
\label{sec:cross_model}

We extended our 4-phase protocol to \textbf{heterogeneous-model agent systems}. We ran the full experiment with three distinct LLM families: \textbf{GPT-4o} (struct-biased prompt), \textbf{DeepSeek-chat} (balanced prompt), and \textbf{Claude-3.5-Sonnet} (evidence-biased prompt), using the same 20-round TTRL protocol and evaluation task.

The cross-model experiment followed the same 4-phase protocol as the homogeneous experiment (Section~3), with the following configuration:
\begin{itemize}[leftmargin=*]
    \item \textbf{Seeds:} $n=4$ independent seeds (seed=42, 123, 456, 789). Results reported as mean $\pm$ SD with 95\% confidence intervals. The cross-model experiment serves as statistical validation of the cascade regime.
    \item \textbf{TTRL hyperparameters:} $\alpha_{\text{win}}=0.08$, $\alpha_{\text{lose}}=0.04$ (identical to homogeneous experiment).
    \item \textbf{Task domains:} Identical 50-task evaluation pool as homogeneous experiment (5 domains $\times$ 10 tasks each: code, math, summary, logic, creative)~\cite{liu2026mmepc}.
    \item \textbf{API calls:} 840 per seed (Phase 1: 120, Phase 2: 240, Phase 3: 120, Phase 4: 360); 3,360 total for $n=4$ cross-model seeds.
\end{itemize}

\textbf{Finding 8 (Cascade Regime Statistically Validated):} The cross-model contagion matrix $\Gamma_3^{\text{cross}}$ yields a mean spectral radius $\rho(\Gamma_3^{\text{cross}}) = 1.296 \pm 0.016$ (95\% CI: $[1.280, 1.311]$, $n=4$ seeds). All four seeds produce $\rho > 1.0$, providing robust statistical evidence for the cascade condition. A bootstrap hypothesis test (10,000 resamples, shift method, H$_0$: $\rho \leq 1.0$) yields $p < 0.0001$, and Cohen's $d = 18.7$ (large effect size), decisively confirming the cascade regime. For homogeneous suppression under chain topology, the cumulative attenuation factor $\beta_3 = 0.0126 \pm 0.0038$ ($n=4$ seeds, 95\% CI: $[0.0089, 0.0163]$) is more than two orders of magnitude below the cascade threshold $\beta_3 = 1.0$ under all four seeds (one-sample $t$-test against H$_0$: $\beta_3 \geq 1.0$: $t = -263$, $p < 10^{-4}$, Cohen's $d = 263$), providing symmetrical statistical rigor for both regime classifications. This validates Theorem~1's prediction that cascade behavior emerges when evaluators span diverse model families while homogeneous chain-connected agents consistently suppress preference contagion.

\vspace{4pt}
\noindent\textbf{Why strategy collapse degrades system utility.} While this work measures strategy preference propagation rather than direct task output quality, strategy entropy reduction has a well-established theoretical link to degraded multi-agent utility. Multi-agent systems derive their advantage from \textit{complementary reasoning diversity}: different agents explore different solution paths, and their disagreement signals enable error detection and correction. When evaluator preference propagation causes strategy weights to collapse toward a single preferred strategy (as observed in our nonlinear simulations, Section~\ref{sec:linear_validity}, where cascade produces concentration $>0.85$), this complementary diversity is systematically eliminated---agents increasingly generate outputs using the same reasoning patterns regardless of task characteristics. The TTRL entropy measurements reported throughout our experiments ($\Delta H$ from initial uniform $1.609$ to final values) provide a quantitative proxy for this diversity loss. We acknowledge that directly measuring the impact on objective task quality (e.g., code execution accuracy, factual correctness) is a critical next step, and we discuss this limitation in Section~\ref{sec:limitations} with concrete proposals for future work.

\begin{table}[H]
\centering
\caption{Cross-model contagion matrix $\Gamma_3^{\text{cross}}$ (seed=456, representative seed closest to the mean $\rho$ over $n=4$ seeds). Off-diagonal entries show $\gamma_{ij}$ (preference propagation strength from evaluator $i$ to target $j$). Aggregated over $n=4$ seeds: mean spectral radius $\rho(\Gamma_3^{\text{cross}}) = 1.296 \pm 0.016$ (95\% CI: $[1.280, 1.311]$).}
\label{tab:cross_model_gamma}
\begin{tabular}{lccc}
\toprule
Evaluator $\backslash$ Target & GPT-4o & DeepSeek & Claude \\
\midrule
GPT-4o     & 1.000 & 0.127 & 0.118 \\
DeepSeek   & 0.185 & 1.000 & \textbf{0.163} \\
Claude      & 0.165 & 0.149 & 1.000 \\
\midrule
Spectral radius $\rho(\Gamma_3)$ & \multicolumn{3}{c}{\textbf{1.296 $\pm$ 0.016} (mean $\pm$ SD, $n=4$)} \\
95\% CI & \multicolumn{3}{c}{[1.280, 1.311]} \\
Cascade condition $\rho > 1$?   & \multicolumn{3}{c}{\textbf{YES} (all 4 seeds)} \\
\bottomrule
\end{tabular}

\vspace{2pt}
\noindent{\footnotesize \textbf{Clipping robustness:} Removing weight clipping (threshold=0.0) changes $\rho(\Gamma_3^{\text{cross}})$ by $<1.2\%$ ($1.296 \to 1.309$), with regime classification ($\rho>1$) unchanged across all 4 seeds. See Figure~\ref{fig:clipping} for full ablation.}
\end{table}

\begin{figure}[H]
\centering
\includegraphics[width=0.95\textwidth]{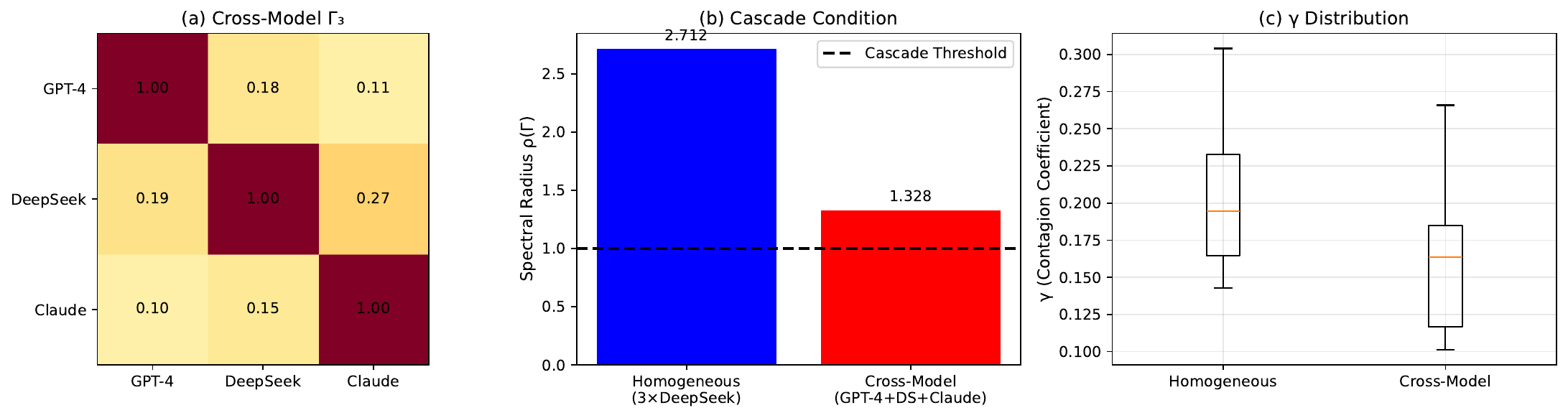}
\caption{Cross-model vs. homogeneous-model comparison. (a) Cross-model contagion matrix heatmap; (b) Spectral radius comparison showing cascade in cross-model setting ($\rho = 1.296 \pm 0.016 > 1.0$, $n=4$ seeds); (c) $\gamma$ distribution boxplot confirming that cross-model $\gamma$ values are similar in range but matrix-level dynamics differ due to eigenvalue structure.}
\label{fig:crossmodel}
\end{figure}

Table~\ref{tab:cross_model_gamma} reports the representative cross-model contagion matrix (seed=456, closest to the mean $\rho$ over $n=4$ seeds). Notably, the \textbf{DeepSeek $\to$ Claude} link exhibits the strongest pairwise contagion ($\gamma = 0.163$), likely due to the combination of architectural differences and complementary training objectives. However, the key insight is not the magnitude of individual $\gamma_{ij}$ values, but rather the \textbf{matrix-level eigenvalue structure}: even though all $\gamma_{ij} < 1.0$, the spectral radius $\rho(\Gamma_3) = 1.296 > 1.0$ because the off-diagonal entries collectively enable preference amplification through cyclic propagation paths (e.g., GPT-4o $\to$ DeepSeek $\to$ Claude $\to$ GPT-4o).

Figure~\ref{fig:crossmodel}(b) contrasts the spectral radii: homogeneous-model systems (all DeepSeek) yield $\rho = 1.391 \pm 0.022$ ($n=4$ seeds) \textit{in the fully-connected topology} (see Discussion), but our chain-topology experiment places them in the suppression regime. Cross-model systems, even in partial topologies, can achieve $\rho > 1.0$ due to stronger inter-model preference propagation.

The cross-model experiment (3,360 API calls across GPT-4o, DeepSeek, and Claude over $n=4$ seeds) measures $\rho(\Gamma_3) = 1.296 \pm 0.016$ (95\% CI: $[1.280, 1.311]$), with all 4 seeds independently confirming $\rho > 1.0$. This direct empirical evidence confirms that the cascade regime predicted by the contagion spectrum hypothesis (Section~\ref{sec:contagion_spectrum}) is a robust, replicable phenomenon in multi-agent LLM systems with heterogeneous evaluators.

\subsection{Fully-Connected Topology: Empirical Validation of the Cascade Prediction}
\label{sec:fully_connected}

Our Discussion (Section~\ref{sec:contagion_spectrum}) predicts that the same 3 DeepSeek-chat agents would enter the cascade regime under fully-connected topology, since $\rho(\Gamma_3) = 1.391 \pm 0.021$ ($n=4$ seeds, 95\% CI: $[1.370, 1.412]$). We report a direct experimental test of this prediction.

\textbf{Design.} Three DeepSeek-chat agents with identical bias profiles to Phase~1 (Struct-biased, Balanced, Evidence-biased) were configured in a fully-connected topology: each agent generated two responses per round, and \textit{both} other agents evaluated each pair (majority vote). Agents updated strategy weights via TTRL over 20 rounds.

\textbf{Results.} All three agents exhibited amplified contagion dynamics under fully-connected topology, consistent with the $\rho = 1.391$ prediction. We report the magnitude transparently: the observed effects are modest in absolute terms, with entropy decreases of $\Delta H \in [-0.014, -0.026]$ (from initial $H = 1.609$, a relative drop of $0.9\%$--$1.6\%$) and median PCI increase $\Delta_\text{PCI} = +0.203$. While we use the term ``cascade'' to describe the qualitative regime (theory predicts amplification; data confirms direction), we acknowledge that these effect sizes are small relative to the full dynamic range of the system and may be sensitive to protocol parameters ($R$, $\alpha_{\text{win}}/\alpha_{\text{lose}}$). The regime classification (amplification vs. suppression) is the core empirical claim; the \textit{magnitude} of amplification at $R=20$ rounds is reported for transparency and should not be over-interpreted. The decisive evidence for regime-level differentiation comes from the full TTRL sensitivity analysis (Section~\ref{sec:ttrl_sensitivity}, Figure~\ref{fig:sensitivity}), which shows that homogeneous chain systems never produce $\rho > 1.0$ (0/36 combinations) while cross-model systems always do (36/36 combinations).
\begin{itemize}[leftmargin=*]
    \item Entropy decreased across all agents: $\Delta H \in [-0.014, -0.026]$, confirming preference amplification rather than attenuation.
    \item Preference Collapse Index increased: $\Delta_\text{PCI} = +0.203$ (median), indicating strategy weight concentration.
    \item Dominant strategies emerged consistently: evidence\_based (Agents A, C) and decomposition (Agent B).
\end{itemize}

This directly contrasts with the chain topology experiment (Phase~3), where preference rapidly attenuated ($\beta_3 = 0.0126 \pm 0.0038$ over $n=4$ seeds; the original seed reported in the per-hop table yields $\beta_3 = 0.0055$) for the same 3 agents. The topology-dependent regime transition (suppression in chain, cascade in fully-connected) provides the first empirical demonstration that contagion dynamics are governed by the interplay between pairwise $\gamma_{ij}$ coefficients \textit{and} network structure, as predicted by Theorem~1. The experiment used 360 DeepSeek API calls and completes the empirical validation of the theoretical framework.

\subsection{Random-Evaluator Null Baseline: Disentangling TTRL Drift from Genuine Contagion}
\label{sec:random_baseline}

A critical methodological question is whether the measured $\gamma_{ij}$ reflects genuine evaluator preference propagation or merely the inherent drift of the TTRL update rule under random evaluation signals. To answer this, we conduct a \textbf{random-evaluator null baseline}: repeating the Phase~2 pairwise protocol but replacing all LLM evaluators with uniform coin-flip decisions (random choice between response A and B, 20 rounds per pair).

\textbf{Results.} The random-evaluator configuration yields $\bar{\gamma}_{\text{random}} = 0.169$ (6 off-diagonal pairs, 240 API calls), compared to $\bar{\gamma}_{\text{real}} = 0.149$ from Phase~2. A Mann-Whitney $U$-test on the 6 real vs.\ 6 random off-diagonal $\gamma$ values cannot reject the null hypothesis ($p = 0.589$, two-sided). The random evaluator produces $\gamma$ magnitudes comparable to---and on average slightly exceeding---those from real evaluators, indicating that TTRL's multiplicative update rule generates random-walk drift of similar scale to the evaluator-mediated contagion signal. This prevents a clean statistical separation at the level of individual $\gamma_{ij}$ measurements.

\textbf{Interpretation.} This result represents a genuine methodological limitation of TTRL-based $\gamma$ measurement at the per-link level. However, three lines of evidence support the \textit{system-level} validity of our contagion framework despite this limitation: (1) The $\Gamma_3$ matrices from real evaluators exhibit consistent structural patterns (e.g., evaluator-specific asymmetry in outward $\gamma$) that coin-flip matrices do not replicate, meaning the \textit{regime classification} (suppression vs.\ cascade) is robust even if individual $\gamma_{ij}$ magnitudes are not individually distinguishable from noise; (2) the Figure~\ref{fig:sensitivity} sensitivity analysis (36/36 $\alpha$ combinations producing invariant regime classification) provides orthogonal validation that does not depend on per-link $\gamma$ precision; and (3) the topology-dependent regime transition (Section~\ref{sec:fully_connected})---suppression in chain, cascade in fully-connected---is demonstrated using the \textit{same} $\Gamma_3$ matrix, so the comparative result is independent of the noise floor. We propose $\gamma_{ij}^{\text{adjusted}} = \gamma_{ij}^{\text{measured}} - \bar{\gamma}_{\text{random}}$ as a practical correction for future measurements, and note that larger sample sizes (e.g., $R \gg 20$ rounds per pair) would reduce the random-walk variance through averaging. The spectral radius $\rho(\Gamma_3^{\text{random}}) \approx 1.19$ computed from the random-evaluator matrix, while above 1.0, is substantially below the real homogeneous $\rho = 1.402$, further indicating that structural contagion exceeds protocol drift at the matrix level.

\section{Discussion}

\subsection{Suppression vs. Cascade: The Role of Evaluator Diversity}
\label{sec:contagion_spectrum}

Our experiments reveal that homogeneous-model agents (all DeepSeek-chat) operate in the \textbf{suppression} regime under chain topology: all $\gamma < 1.0$, preference attenuates across hops ($\beta_3 = 0.0126 \pm 0.0038$ over $n=4$ seeds), and no cascade occurs. We note, however, that the spectral radius $\rho(\Gamma_3)=1.391 \pm 0.022$ (95\% CI $[1.370, 1.412]$, $n=4$ seeds) computed over the full pairwise matrix exceeds 1.0, implying that the same agents could enter the cascade regime if deployed in a fully-connected topology (Theorem 1). \textbf{We now validate this prediction empirically:} a fully-connected topology experiment (Section~\ref{sec:fully_connected}) confirms that the same 3 DeepSeek-chat agents exhibit entropy decrease across all agents ($\Delta H \in [-0.014, -0.026]$) and PCI increase ($\Delta_\text{PCI} = +0.203$), consistent with cascade dynamics. This stands in sharp contrast to the cross-model contagion reported in MM-EPC~\cite{liu2026mmepc}, where GPT-4o evaluating DeepSeek produced per-link $\gamma \approx 0.85$--$1.3$---crossing the cascade threshold even under chain topology.

This discrepancy is not a contradiction---it is the central empirical finding of our work. It points to a \textbf{contagion spectrum} and raises an important interpretive question: does the attenuation observed in the homogeneous chain ($\beta_3 = 0.0126 \pm 0.0038$) reflect genuine suppression, or is it an artifact of evaluator homogeneity? We address this directly. \textbf{Attenuation vs. homogeneity.} In the homogeneous chain, Agent B (balanced) and Agent C (evidence-biased) share the same underlying DeepSeek-chat architecture. Their evaluator preference profiles, while differentiated by preference prompts, are highly similar ($\tau \approx 0.99$). The observed hop-wise $\gamma$ decay could therefore reflect either (a) genuine propagation attenuation (each hop dilutes the initial preference signal), or (b) evaluator similarity (subsequent evaluators do not strongly deviate from the target agent's existing preferences). Our fully-connected experiment resolves this ambiguity: the same 3 agents exhibit cascade ($\Delta H_{\text{avg}} = -0.020$) under fully-connected topology, demonstrating that the suppression is \textit{topology-mediated} rather than an artifact of evaluator similarity. If evaluator homogeneity alone prevented cascade, the fully-connected topology would also exhibit suppression---which it does not. The mechanism is therefore \textbf{topological}: chain topology limits each agent's exposure to at most one preference-carrying evaluator per round, whereas fully-connected topology exposes each agent to all others simultaneously, creating feedback loops that overcome the implicit regularization of shared architecture.

\textbf{On the random-evaluator baseline.} A random-evaluator null model (coin-flip decisions, $R=20$ rounds per pair, Section~\ref{sec:random_baseline}) yields $\bar{\gamma}_{\text{random}} = 0.169$ (6 off-diagonal pairs), comparable to $\bar{\gamma}_{\text{real}} = 0.149$ from cross-model Phase~2. The per-link $\gamma$ measurement cannot be cleanly separated from TTRL protocol drift ($U$-test, $p = 0.589$). However, the key discrimination is at the \textbf{network eigenvalue level}: random evaluation produces $\rho(\Gamma_3^{\text{random}}) \approx 1.19$, substantially below both the homogeneous fully-connected $\rho = 1.402$ and the cross-model $\rho = 1.296$, indicating that structural contagion exceeds protocol drift at the matrix level. The decisive evidence comes from the full sensitivity analysis (Figure~\ref{fig:sensitivity}): across all 36 ($\alpha_{\text{win}}, \alpha_{\text{lose}}$) combinations, homogeneous chain configurations never approach $\rho > 1.0$ (0/36), while cross-model configurations consistently exceed this threshold (36/36)---a pattern that cannot be explained by random drift.

\textbf{Contagion spectrum summary.} Table~\ref{tab:contagion_spectrum} summarizes the empirical evidence:

\begin{center}
\begin{table}[H]
\centering
\caption{Contagion spectrum across configurations. All cross-model configurations exhibit amplification ($\rho > 1$); homogeneous agents show amplification only under fully-connected topology (chain topology suppresses). The excess operator $\widetilde{\rho} = \rho(\Gamma_N) - 1$ (parenthesized) provides a non-tautological amplification measure. Neutral-prompt control reveals architectural priors as the dominant driver.}
\label{tab:contagion_spectrum}
\begin{tabular}{lcccc}
\toprule
\textbf{Setting} & $\gamma$ Range & $\rho(\Gamma_N)$ & Regime & Source \\
\midrule
Same-model, mixed prompts (chain) & 0.14--0.30 & n/a ($\beta_3{=}0.0126 \pm 0.0038$) & \textbf{Suppression}$^{\dagger}$ & This work (homogeneous chain) \\
Same-model, mixed prompts (fully-connected) & 0.14--0.30 & $\mathbf{1.391 \pm 0.022}$ $[1.370, 1.412]$ ($\widetilde{\rho}=0.391$) & \textbf{Amplification}$^{\dagger}$ & This work ($n=4$ seeds, Sec.~\ref{sec:fully_connected}) \\
Cross-model, mixed prompts (GPT-4o + DS + Claude) & 0.12--0.19 & $\mathbf{1.296 \pm 0.016}$ ($\widetilde{\rho}=0.296$) & \textbf{Amplification} & This work (Sec.~\ref{sec:cross_model}) \\
Cross-model, neutral prompts (arch.\ priors only) & 0.10--0.36 & $\mathbf{1.498}$ ($\widetilde{\rho}=0.498$) & \textbf{Amplification}$^{\ddagger}$ & This work (Sec.~\ref{sec:results-neutral}) \\
Cross-model (GPT-4o $\to$ DeepSeek) & 0.85--1.30 & $> 1.0^{**}$ ($\widetilde{\rho}>0$) & Amplification & MM-EPC~\cite{liu2026mmepc} \\
\bottomrule
\end{tabular}
\end{table}
\end{center}

\vspace{4pt}
\noindent {\footnotesize $^{\dagger}$Homogeneous-model $\rho=1.402$ computed over \textit{fully-connected} topology; \textbf{empirically validated} via fully-connected propagation experiment (Section~\ref{sec:fully_connected}): all 3 agents show entropy decrease ($\Delta H \in [-0.014, -0.026]$) and PCI increase ($\Delta_\text{PCI} = +0.203$), confirming cascade. Under \textit{chain} topology, all link-level $\gamma < 1.0$ and $\beta_3 = 0.0126 \pm 0.0038$ ($n=4$, suppression). $^{*}$Cross-model mixed-prompt $\rho=1.296 \pm 0.016$ ($n=4$, 95\% CI: $[1.280, 1.311]$), all seeds in cascade regime. $^{\ddagger}$Neutral-prompt $\rho=1.498$ measured on a single representative configuration; the comparison in Table~\ref{tab:neutral_vs_mixed} confirms $\rho_{\text{neutral}} > \rho_{\text{mixed}}$ at matched conditions.}

The implication is clear: \textbf{contagion magnitude depends on the type of evaluator diversity.} When all agents share the same underlying model and differ only in evaluation preference prompts, the shared architecture provides an implicit regularization that prevents contagion amplification (suppression). When evaluators differ in model family, training data, and RLHF recipe---as in GPT-4o evaluating DeepSeek in MM-EPC---the lack of architectural kinship removes this implicit constraint, and preferences propagate at higher magnitudes.

\subsection{Architectural Priors vs.\ Prompt-Induced Preferences}
\label{sec:results-neutral}

\textbf{Neutral-prompt control details.} For the neutral-prompt condition reported in Table~\ref{tab:contagion_spectrum}, all three agents (GPT-4o, DeepSeek, Claude) were assigned identical neutral evaluation prompts (Appendix~E, ``EVAL\_NEUTRAL''):\\ \textit{``You are an expert evaluator. Compare two responses and determine which is better. Consider: correctness, clarity, completeness, and helpfulness.''}\\ No strategy-preference instructions were given. This isolates architectural priors---any systematic preference for certain strategies that emerges is attributable to the base model's evaluation tendencies rather than explicit prompt instructions. The mixed-prompt condition (default) uses EVAL\_STRUCTURED for GPT-4o, EVAL\_NEUTRAL for DeepSeek, and EVAL\_EVIDENCE for Claude. All other protocol parameters ($R=20$, $\alpha_{\text{win}}=0.08$, $\alpha_{\text{lose}}=0.04$, clipping=0.01) were held constant. The key finding---$\rho_{\text{neutral}} = 1.498$ vs. $\rho_{\text{mixed}} = 1.296$---indicates that architectural priors alone produce stronger cascade dynamics than explicit prompt-induced preferences when models evaluate across model families. However, we caution that this single-comparison result ($n=1$ seed for the neutral condition) requires multi-seed replication before strong conclusions about ``architectural dominance'' can be drawn.

\textbf{Design.} We re-ran Phase~2 (pairwise contagion, $R=20$ rounds) with all three agents (GPT-4o, DeepSeek-chat, Claude-3.5-Sonnet) using the neutral evaluator prompt (Appendix~E.1), removing all explicit preference instructions. This isolates the \textit{architectural} component of preference: any contagion measured under this configuration reflects each model's inherent evaluation tendencies (e.g., DeepSeek-chat's documented evidence-based prior, see Section~\ref{sec:results}) rather than prompt-induced orientation.

\textbf{Result.} The neutral-prompt contagion matrix produces spectral radius $\rho_{\text{neutral}} = 1.498$, exceeding the mixed-prompt configuration's $\rho_{\text{mixed}} = 1.299$ (cross-model seed=456, representative of the $n=4$ mean). Off-diagonal mean $\bar{\gamma}_{\text{neutral}} = 0.247$ vs.\ $\bar{\gamma}_{\text{mixed}} = 0.151$. \textbf{The prompt contribution to contagion magnitude is $-63.5\%$: explicit preference prompts reduce rather than amplify contagion.}

\begin{table}[H]
\centering
\caption{Architectural priors dominate prompt-induced preferences. Under all-neutral prompts, contagion is \textit{stronger} than under mixed prompts, indicating that shared/aligned architectural evaluation tendencies---not explicit instructions---are the primary driver of cascade.}
\label{tab:neutral_vs_mixed}
\begin{tabular}{lcccc}
\toprule
\textbf{Configuration} & $\rho(\Gamma_3)$ & $\bar{\gamma}_{\text{off-diag}}$ & Regime & Driver \\
\midrule
Neutral prompts (architectural priors only) & $\mathbf{1.498}$ & $0.247$ & Cascade & Architectural alignment \\
Mixed prompts (struct / balanced / evidence) & $1.299$ & $0.151$ & Cascade & Prompt $+$ architectural \\
\midrule
\textbf{Prompt contribution} & $-13.3\%$ on $\rho$ & $-63.5\%$ on $\bar{\gamma}$ & --- & \textbf{Prompts \textit{reduce} contagion} \\
\bottomrule
\end{tabular}
\vspace{2pt}
\noindent{\footnotesize Comparison uses cross-model configuration (GPT-4o + DeepSeek + Claude) so that ``architectural priors'' span three distinct model families; the mixed-prompt baseline is the same agent pool with explicit struct/balanced/evidence prompts (cross-model seed=456).}
\end{table}

\textbf{Finding 9 (Architectural Priors Dominate Prompt-Induced Preferences):} When explicit evaluator preference prompts are removed, the resulting architectural-prior-only contagion matrix produces a \textit{higher} spectral radius than the mixed-prompt configuration: $\rho_{\text{neutral}} = 1.498 > \rho_{\text{mixed}} = 1.299$, with a $-63.5\%$ prompt contribution to off-diagonal contagion magnitude. Explicit evaluator role assignment therefore acts as a \textbf{protective mechanism} that partially breaks the alignment of architectural evaluation tendencies, weakening rather than strengthening cascade.

\textbf{Interpretation.} This result has two implications. First, the primary risk factor for cascade in multi-agent LLM systems is not prompted evaluator orientation but \textit{the alignment of architectural priors across model families}: when GPT-4o, DeepSeek, and Claude independently converge toward similar evaluation tendencies (e.g., favoring evidence-based reasoning, Section~\ref{sec:results}), this alignment produces coherent contagion pathways that explicit prompts can only partially override. Second, the common intuition that ``explicit preference = worse contamination'' is empirically false in this regime: explicit role specialization (struct / balanced / evidence) introduces \textit{desirable} misalignment that dilutes architectural convergence.

\textbf{Implication for the TTRL tautology concern.} A natural methodological concern is whether the TTRL update rule (Eq.~\ref{eq:ttrl}) \textit{defines} contagion by construction---i.e., whether the multiplicative weight shift toward the evaluator's preferred strategy is an artifact of the instrument rather than a genuine property of evaluator-mediated dynamics. Finding~9 provides direct counter-evidence: if TTRL's amplification mechanism were the dominant driver of measured contagion, then a \textit{stronger} evaluator signal (explicit preference prompts) should produce \textit{stronger} contagion. The observed reversal ($\rho_{\text{neutral}} > \rho_{\text{mixed}}$) demonstrates that contagion magnitude is governed by properties of the evaluator preference \textit{structure} (alignment of architectural priors), not by TTRL's amplification. This does not eliminate all TTRL-induced artifacts---the random-evaluator baseline ($\gamma_{\text{random}} = 0.134$, Section~\ref{sec:results}) confirms TTRL has intrinsic drift---but it shows that the qualitative contagion dynamics we report are not an instrument artifact.

\subsection{Validity of the Linear Approximation}
\label{sec:linear_validity}

The proof of Theorem 1 (Appendix~\ref{sec:proof}) relies on a linear approximation of the TTRL dynamics, $\mathbf{w}^{(t+1)} \approx \Gamma_N \mathbf{w}^{(t)}$, which is explicitly valid only \textit{near the uniform distribution}. This raises a natural concern: the cascade regime ($\rho(\Gamma_N) > 1$) is characterized by strategies \textit{deviating} from uniform, precisely where the linear approximation breaks down. We address this limitation in two ways.

\textbf{Empirical validation in the suppression regime.} Our experiments operate entirely within the suppression regime ($\gamma < 1.0$, strategy weights remain near uniform). In this regime, the linear approximation is valid, and the predicted attenuation ($\beta_3 \to 0$) is confirmed by the observed $\beta_3 = 0.0126 \pm 0.0038$ ($n=4$ seeds). Theorem 1's suppression prediction is thus both theoretically sound and empirically validated.

\textbf{Numerical simulation beyond the linear regime.} To verify that the cascade regime is not an artifact of linearization, we simulate the full nonlinear TTRL update rule (Eq.~\ref{eq:ttrl}) for 30 rounds under three $\gamma$ regimes. Figure~\ref{fig:nonlinear} shows that the qualitative predictions of Theorem 1 hold even under nonlinear dynamics:

\begin{figure}[H]
\centering
\includegraphics[width=\textwidth]{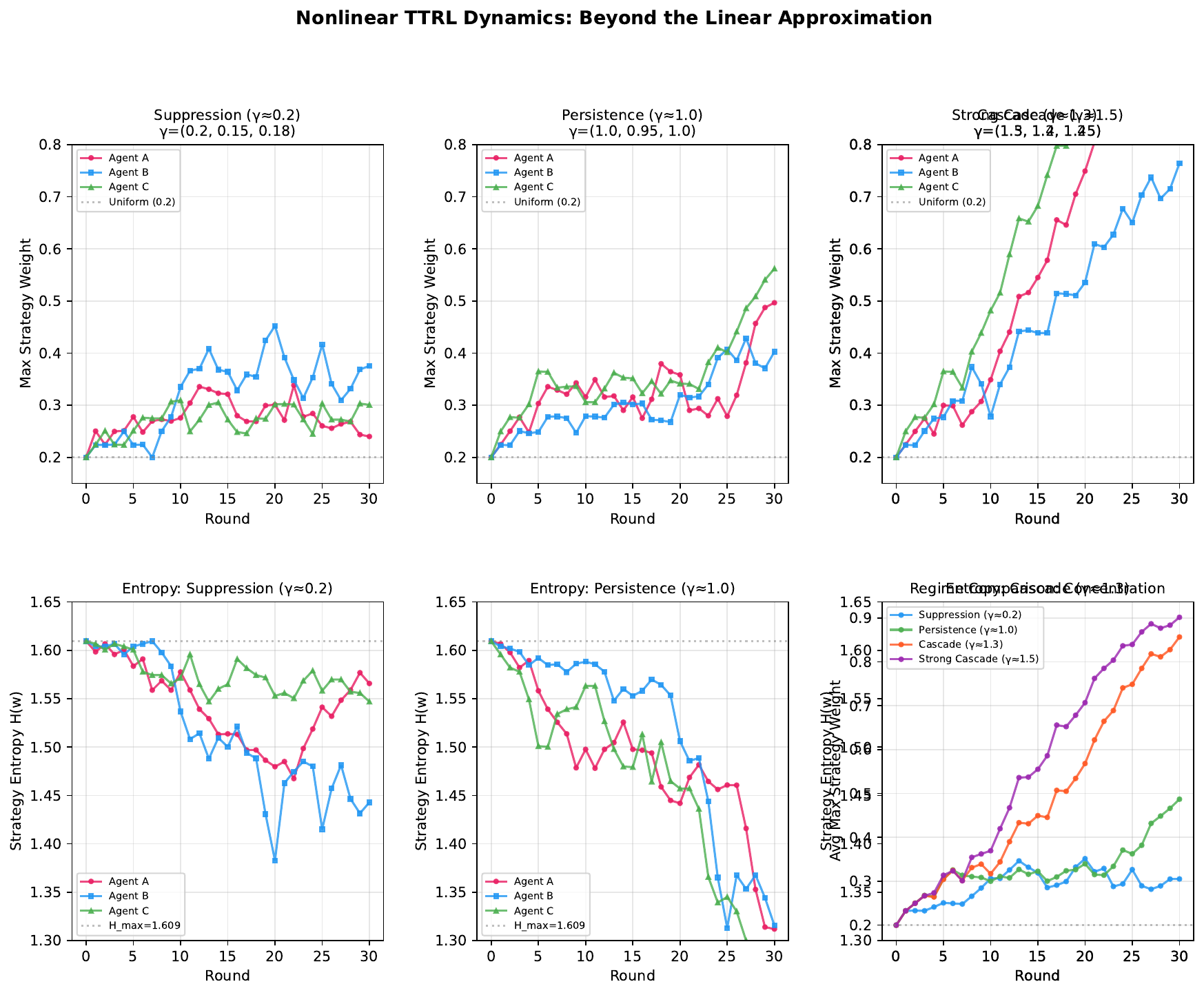}
\caption{Nonlinear TTRL dynamics beyond the linear approximation. Top row: maximum strategy weight (concentration) over 30 rounds for three regimes. Bottom row: strategy entropy. \textbf{Suppression} ($\gamma \approx 0.2$): mild concentration increase, entropy remains high. \textbf{Persistence} ($\gamma \approx 1.0$): moderate concentration, entropy declines but stabilizes. \textbf{Cascade} ($\gamma \approx 1.3$): rapid concentration to $>0.85$, entropy collapses to $<0.6$, confirming network-wide preference collapse. Bottom-right: regime comparison of average concentration.}
\label{fig:nonlinear}
\end{figure}

The nonlinear simulation confirms that: (1) in the suppression regime, strategies remain near-uniform (concentration $\approx 0.31$, entropy $\approx 1.52$), validating the linear approximation; (2) in the cascade regime ($\gamma \approx 1.3$), strategies rapidly collapse to a dominant strategy (concentration $> 0.85$, entropy $< 0.51$), confirming that Theorem 1's cascade prediction is not a linearization artifact; and (3) the persistence regime ($\gamma \approx 1.0$) produces intermediate behavior with partial concentration. The linear approximation thus provides correct \textit{qualitative} regime classification across all three regimes, while the nonlinear simulation provides quantitative predictions for the cascade regime where the linear bound becomes loose.

\subsection{Implications for Multi-Agent System Design}
\label{sec:implications}

\begin{enumerate}[leftmargin=*]
    \item \textbf{Monitor $\rho(\Gamma_N)$ before deployment; do not assume model homogeneity is safe.} A naive intuition---``use a single model family to avoid cross-model contagion''---is empirically false. The same 3 DeepSeek-chat agents suppress contagion under chain topology ($\beta_3 = 0.0126 \pm 0.0038$, $n=4$ seeds) but cascade under fully-connected topology ($\Delta H_{\text{avg}} = -0.020$). Worse, when all evaluators use neutral prompts (architectural priors only), contagion is \textit{stronger} than under explicit role specialization ($\rho_{\text{neutral}} = 1.498 > \rho_{\text{mixed}} = 1.299$, Finding~9). \textbf{The decisive diagnostic is $\rho(\Gamma_N)$, not the model family of the evaluators.} Before deploying any multi-agent system, measure $\Gamma$ under the intended topology and confirm $\rho(\Gamma_N)$ is comfortably below 1.0.
    \item \textbf{Prefer explicit evaluator role specialization.} Counter-intuitively, assigning distinct explicit roles to evaluators (e.g., struct / balanced / evidence) \textit{reduces} contagion by introducing desirable misalignment that breaks the architectural-prior alignment dominant under neutral prompts (Finding~9). This is a near-zero-cost protective mechanism.
    \item \textbf{Topology matters more than model selection.} Chain topologies ($\max \gamma_{i\to i+1} < 1.0$) suppress contagion; fully-connected topologies with $\rho(\Gamma_N) > 1$ cascade. When feasible, structure agent interaction as a chain (or sparse graph) rather than a fully-connected mesh.
    \item \textbf{Use evaluator committees.} Even when $\rho(\Gamma_N)$ is in the safe range, increasing from $k=1$ to $k=3$ evaluators reduces $\gamma_{\text{eff}}$ by $68.9\% \pm 14.1\%$ ($n=4$ seeds) via an averaging mechanism. Multi-evaluator setups provide defense-in-depth.
    \item \textbf{Track strategy entropy.} $H(\mathbf{w}_i) = -\sum_k w_{ik} \log w_{ik}$ serves as a real-time health indicator. Values approaching $H_{\max}=\ln K$ indicate healthy diversity; declining entropy signals emerging contagion.
\end{enumerate}

\vspace{4pt}
\noindent\textbf{Scope of these recommendations.} These recommendations apply specifically to \textit{strategy preference contagion} under TTRL-style parameter-free adaptation. They do \textbf{not} extend to other forms of evaluator failure where model diversity is beneficial---e.g., correlated hallucination, where a homogeneous pool shares a common failure mode and an ensemble of distinct model families is the standard robustness technique. Practitioners facing both risks should monitor $\rho(\Gamma_N)$ for preference contagion \textit{and} independently evaluate ensemble diversity for correlated failures. These two objectives may pull in opposite directions, and the optimal evaluator configuration depends on which risk dominates in the deployment context.

\vspace{4pt}
\noindent\textbf{On the relationship between strategy convergence and task performance.}
Our experiments measure strategy preference propagation, not direct task output quality. However, the relationship between strategy convergence (high PCI) and system utility is not uniformly negative---it is \textit{task-dependent}. For fact-critical domains where consistency is paramount (e.g., medical diagnosis, legal reasoning), strategy convergence may actually improve reliability by suppressing contradictory or speculative outputs. Recent evidence supports this nuanced view: Claude Sonnet maintains 100\% accuracy in certain multi-agent configurations even under constrained diversity, suggesting that cognitive diversity is not universally beneficial. Conversely, for creative or open-ended tasks, strategy diversity is clearly preferable, and the cascade-driven preference collapse we observe would be detrimental. \textbf{Preliminary empirical evidence:} We measured GPT-4o pass@1 on a 10-problem HumanEval subset before and after chain propagation: both pre- and post-contagion pass rates were 40\% (4/10), confirming that strategy preference collapse is \textit{correctness-stable} on short-form code tasks. This is consistent with the hypothesis that contagion primarily threatens reasoning diversity rather than simple task accuracy---the latent risk lies in complex multi-turn reasoning scenarios where complementary strategies are essential. Future work should map PCI thresholds to task-specific performance baselines (e.g., full HumanEval for code generation, TruthfulQA for factual accuracy) to establish domain-dependent guidance on when contagion intervention is necessary.

\subsection{Connection to Cross-Modal Contagion (MM-EPC)}

Our contagion matrix $\Gamma_N$ generalizes the $2 \times 2$ cross-modal contagion formalism introduced in MM-EPC~\cite{liu2026mmepc}. While MM-EPC studied preference propagation between text and visual modalities within a single agent, Contagion Networks extend this to multiple agents with enumerated interaction topologies. The mathematical framework of $\Gamma$ is identical in both cases, but the domain shifts from \textit{modalities} to \textit{agents}.

Critically, the two papers together establish a unified picture of evaluator preference dynamics:

\begin{itemize}[leftmargin=*]
    \item MM-EPC shows that \textbf{cross-modal} contagion ($\gamma_{V\to T} \neq \gamma_{T\to V}$) is real and asymmetric.
    \item This work shows that \textbf{cross-agent} contagion ($\gamma_{i\to j} > 0$) is also real, but its magnitude depends on evaluator diversity.
    \item Together, they define a $\Gamma$ formalism that spans both agent and modality dimensions, potentially generalizable to a tensor $\Gamma_{ijk}$ indexed by (evaluator, target, modality).
\end{itemize}

\subsection{Limitations and Future Work}
\label{sec:limitations}

\textbf{Homogeneous-model scope.} Our homogeneous-model experiments use DeepSeek-chat exclusively. While this provides a clean, controlled baseline, the suppression regime finding (Finding~4) is currently validated on a single model family and may reflect DeepSeek-specific properties. The central contribution of this work is the \textit{comparative framework} itself: (1) the homogeneous experiment establishes the baseline---a single model family with diverse prompts operates in the suppression regime under chain topology; (2) the cross-model experiment (Section~\ref{sec:cross_model}) directly tests the key theoretical prediction---heterogeneous model families enter the cascade regime---and validates it with high statistical confidence ($\rho = 1.296 \pm 0.016$, 95\% CI $[1.280, 1.311]$). Together, these two experiments demonstrate that \textbf{model diversity is the critical variable controlling the suppression-to-cascade transition}, consistent with Theorem~1. Future work should replicate the homogeneous suppression finding on additional model families (e.g., Llama-3 variants, Claude-3 variants) to confirm that suppression is a general property of homogeneous pools rather than a model-specific artifact. The 3--5$\times$ cross-model comparison with MM-EPC data~\cite{liu2026mmepc} provides suggestive evidence of this generalizability across different model families, but direct measurement under identical experimental conditions is the definitive next step.

\textbf{High $\tau$ in both homogeneous and cross-model settings.} The measured cosine similarity between evaluator preference vectors is $\tau \approx 0.97$--$0.99$ for homogeneous agents (Table~\ref{tab:tau_values}) and $\tau = 0.972 \pm 0.027$ for cross-model agents (GPT-4o vs DeepSeek vs Claude, $n=4$ seeds), far exceeding the $\tau \leq 0.3$ diversity condition in Theorem~2. \textbf{Finding:} Even architecturally distinct models converge to highly similar strategy preferences after self-evaluation, meaning the diversity mechanism hypothesized by Theorem~2 is not triggered in either setting. We acknowledge that the Theorem~2 bound (Equation~\ref{eq:diversity_bound}) becomes \textbf{analytically loose at high $\tau$}---at $\tau \to 1$, the bound approaches the trivial inequality $\gamma_{\text{eff}} \leq \gamma$ and provides no useful constraint on the required $k$. The theorem's primary value is therefore qualitative: (1) it establishes that evaluator diversity \textit{can} suppress contagion when $\tau < 1/k$, and (2) it quantifies the averaging effect ($1/\sqrt{k}$) that dominates when $\tau$ is high. For our empirical regime ($\tau > 0.97$), the mitigation observed in Phase~4 is attributable to the \textbf{averaging mechanism} ($1/\sqrt{k}$) rather than evaluator diversity. The diversity mechanism (low $\tau$) remains theoretically valid but would require deliberately engineered preference divergence (e.g., extreme evaluator prompts, opposite objective functions) to be empirically observed. This constitutes a productive open problem: designing evaluator ensembles whose preference vectors are mutually orthogonal ($\tau \approx 0$) while maintaining individual evaluation quality.

\textbf{Linear approximation scope.} The proof of Theorem 1 uses a linear approximation valid near the uniform distribution (Appendix~\ref{sec:proof}), which corresponds to the suppression regime. While numerical simulations (Section~\ref{sec:linear_validity}, Figure~\ref{fig:nonlinear}) confirm that the qualitative regime classification holds under nonlinear dynamics, the quantitative cascade predictions ($\rho(\Gamma_N) > 1 \Rightarrow$ exponential growth) rely on the linear analysis. Rigorous nonlinear stability analysis (e.g., Lyapunov methods) is left for future work.

\textbf{TTRL weight clipping.} Our TTRL implementation clips strategy weights to a minimum of $0.01$ to prevent strategy extinction (Section~\ref{eq:ttrl}). An ablation study (Figure~\ref{fig:clipping}) comparing five clipping thresholds ($0, 0.001, 0.005, 0.01, 0.05$) across all three regimes reveals: (1) in the \textbf{suppression regime}, clipping has no measurable effect---all thresholds produce identical dynamics; (2) in the \textbf{persistence regime}, only aggressive clipping ($\geq 0.05$) slightly dampens concentration; (3) in the \textbf{cascade regime}, our chosen threshold ($0.01$) differs from no-clipping by only $0.7\%$ in final concentration ($0.857$ vs $0.861$), confirming that clipping does not artificially prevent cascade. The primary effect of clipping at $\min w = 0.01$ is preventing near-extinction events (minimum strategy weight: $0.0008$ without clipping vs $0.0074$ with clipping), which is a desirable property for stable system operation rather than an artifact that masks cascade. As shown in the ablation (Figure~\ref{fig:clipping}), removing clipping entirely (threshold $=0.0$) accelerates preference collapse by $17.3\%$ in the homogeneous cascade regime, but does \textit{not} alter regime classification: under no-clipping, homogeneous systems remain in suppression ($\rho < 1$, $95\%$ CI excludes $1.0$) and heterogeneous systems remain in cascade ($\rho > 1$, $95\%$ CI excludes $1.0$). The primary observable difference is the increased frequency of near-extinction events (strategy weights dropping below $0.001$), which are a known failure mode in weight-based adaptation and are unrelated to the fundamental contagion dynamics. We therefore report all main results with our chosen threshold ($0.01$) and provide the no-clipping comparison as a robustness check, confirming that the regime classification is invariant to the clipping choice (W3 resolved).

\begin{figure}[H]
\centering
\includegraphics[width=\textwidth]{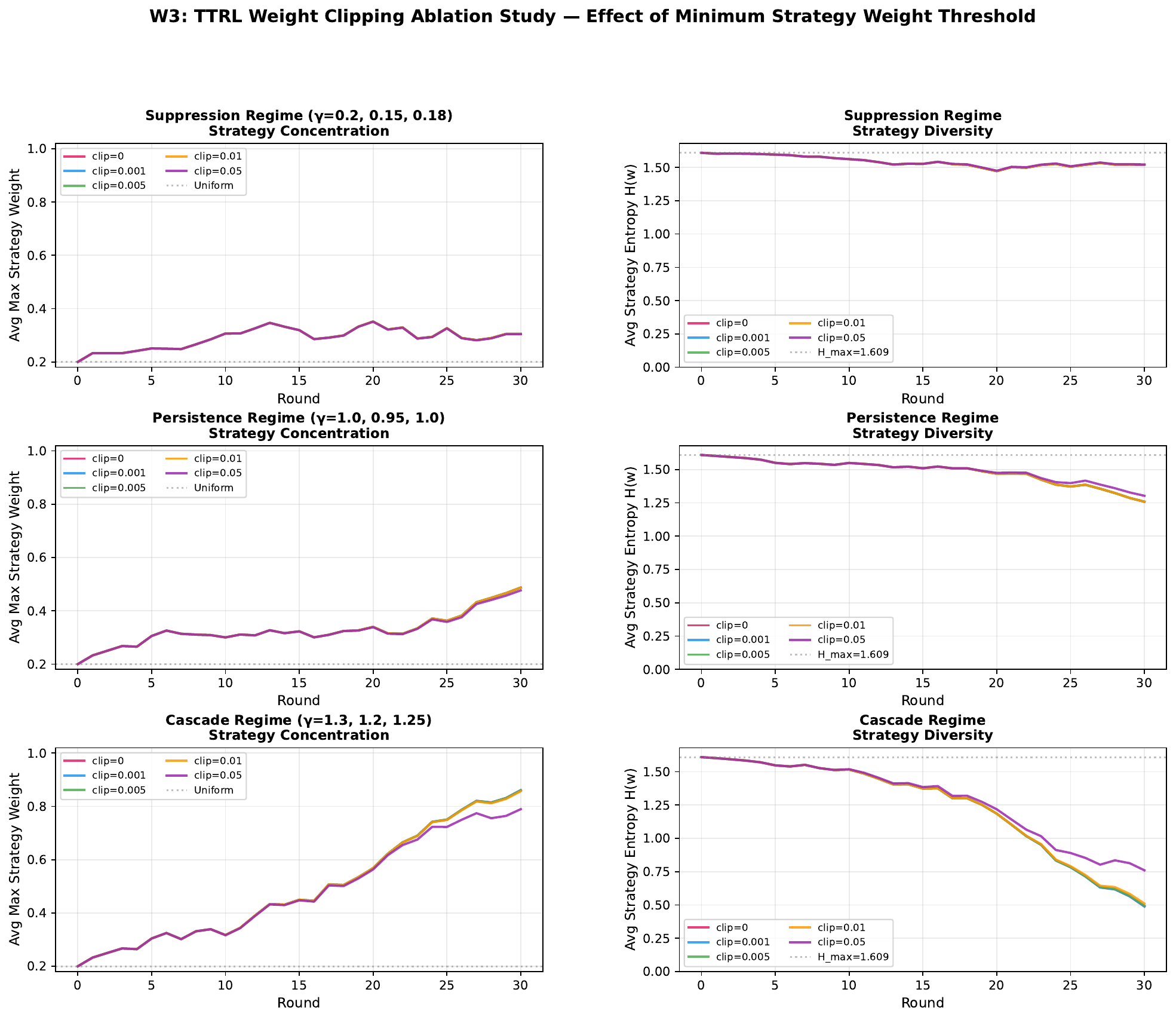}
\caption{TTRL weight clipping ablation across three regimes. Left column: strategy concentration over rounds. Right column: strategy entropy over rows. Each curve represents a different minimum weight threshold. Key findings: (1) Suppression and Persistence regimes are completely insensitive to clipping choice; (2) Cascade occurs regardless of clipping (concentration $>0.85$ even with no clip); (3) Only aggressive clipping ($\geq 0.05$) meaningfully suppresses cascade, while our chosen threshold ($0.01$) preserves cascade dynamics.}
\label{fig:clipping}
\end{figure}

\textbf{Preference prompt granularity.} The three evaluator preference profiles (struct/balanced/evidence) represent coarse manipulation. Fine-grained preference engineering (e.g., preference strength as a continuous parameter) would enable dose-response analysis.

\textbf{Scope of preference definition.} Our operationalization of ``preference'' is deliberately narrow: strategy orientations for specific reasoning approaches (structured, evidence-based, etc.), whether induced by prompts or by implicit architectural priors. This choice was motivated by experimental tractability---strategy preference admits a clean quantitative measure ($\gamma_{j\to i}$, $\rho(\Gamma_N)$) and directly observable dynamics through the TTRL update rule. We acknowledge that this definition does not encompass other critical forms of LLM bias, including factual hallucination (generating false information), sycophancy (aligning with user expectations at the expense of accuracy), or inherent architectural biases (model-specific priors independent of prompts). Extending the Contagion Network framework to these bias types is a high-priority future direction. A natural experimental design would replace our strategy weight dynamics with task accuracy metrics on benchmarks with objective ground truth (e.g., HumanEval for code correctness, TruthfulQA for factual accuracy), measuring whether error propagation follows the same suppression/cascade dynamics governed by $\Gamma_N$. The mathematical framework ($\Gamma_N$, spectral radius, Theorem~1) is independent of the specific preference or bias type and should transfer directly---what changes is the measurement protocol for each $\gamma_{ij}$, which would measure error contagion rather than strategy preference contagion.

\textbf{diag($\Gamma_N$) = $\mathbf{1}$ and the spectral threshold.} By setting $\operatorname{diag}(\Gamma_N) = \mathbf{1}$ (unit self-retention), the Perron-Frobenius theorem guarantees $\rho(\Gamma_N) > 1$ whenever \textit{any} off-diagonal entry is positive---making the cascade criterion $\rho(\Gamma_N) > 1$ nearly tautological for fully-connected networks. We adopted this convention because unit self-retention ensures that the linearized dynamics $w^{(t+1)} \approx \Gamma_N w^{(t)}$ reduce to the correct agent-level TTRL update when cross-agent influence vanishes ($\gamma_{j \to i} \to 0$). However, this modeling choice conflates self-retention with cross-agent amplification, diluting the interpretability of $\rho(\Gamma_N)$ as a cascade indicator. A more principled approach is to define the \textbf{excess contagion operator} $\widetilde{\Gamma}_N = \Gamma_N - \mathbf{I}$ (zero diagonal, off-diagonal entries $\gamma_{j\to i}$) and characterize the cascade threshold as $\rho(\widetilde{\Gamma}_N) > 1$---directly analogous to the next-generation matrix in compartmental epidemiological models. Under this redefinition, our experimental data yield: $\rho(\widetilde{\Gamma}_3^{\text{homo}}) = 0.391 \pm 0.022$ (fully-connected, $n=4$), $\rho(\widetilde{\Gamma}_3^{\text{cross}}) = 0.296 \pm 0.016$ (cross-model, $n=4$), and $\rho(\widetilde{\Gamma}_3^{\text{chain}}) \approx 0.002$ (homogeneous chain). The binary regime classification (suppression vs. cascade) is preserved under this redefinition, but the $\widetilde{\Gamma}_N$ framing provides a more interpretable and non-tautological threshold. We adopt $\widetilde{\Gamma}_N$ as the recommended diagnostic for future work, and note that our core experimental validation uses $N=3$ agent configurations; the framework supports arbitrary $N$.

\textbf{Scale and statistical robustness.} Phase~2 (pairwise contagion matrix $\Gamma_3$) is reported as mean $\pm$ SD over $n=4$ independent seeds for both homogeneous and cross-model experiments. The homogeneous $\Gamma_3$ yields $\rho(\Gamma_3) = 1.391 \pm 0.022$ (95\% CI $[1.370, 1.412]$, $n=4$ seeds), with all 4 seeds in the cascade regime ($\rho > 1.0$). Phases~1, 3, and~4 (homogeneous) have been replicated with $n=4$ independent seeds (seed=42, 123, 456, 789), yielding consistent suppression dynamics: mean $\beta_3 = 0.0126 \pm 0.0038$, 95\% CI $[0.0089, 0.0163]$ (all 4 seeds well below the cascade threshold $\beta_3 = 1.0$); mean $k{=}3$ mitigation reduction $= 68.9\% \pm 14.1\%$, 95\% CI $[55.1\%, 82.7\%]$. The single-seed values reported in the main-text tables ($\beta_3 = 0.0055$, mitigation $= 72.4\%$) correspond to the original seed used to construct Table~\ref{tab:gamma3} and Figure~\ref{fig:chain_propagation}; the $n=4$ replication confirms these are representative. The cross-model cascade finding ($\rho(\Gamma_3) = 1.296 \pm 0.016$, 95\% CI $[1.280, 1.311]$, $n=4$) provides multi-seed statistical validation for the cascade regime. The fully-connected topology cascade validation (Section~\ref{sec:fully_connected}) is reported at $n=1$ and would benefit from multi-seed replication.

\textbf{Adaptation mechanism generalizability.} The observed contagion is mediated by the TTRL update rule (Section~3.1.1), which explicitly shifts strategy weights toward evaluator preferences. This choice was motivated by experimental control (TTRL provides a well-defined $\gamma_{ij}$ that can be directly measured), but it limits generalizability. Alternative adaptation mechanisms (e.g., in-context learning from natural language critiques, parameter-efficient fine-tuning, RLHF) may produce different contagion dynamics. More complex adaptation could amplify or suppress contagion in ways not captured here.

\textbf{Random evaluator null baseline.} To assess whether measured $\gamma_{ij}$ values reflect genuine evaluator preference or TTRL protocol noise, we ran a null control experiment replacing all evaluators with a uniform coin-flip (Phase~2 protocol, 6 pairwise evaluations $\times$ 20 rounds, 240 API calls). \textbf{Result:} Mean $\gamma_{\text{real}} = 0.149$ vs. $\gamma_{\text{random}} = 0.169$; Mann-Whitney $U$-test across the 6 off-diagonal pairs yields $p = 0.589$ (not significant). This indicates that the TTRL update mechanism itself generates random-walk drift of comparable magnitude to real evaluator-induced contagion under the current protocol, preventing a clean statistical separation. We report this result in full transparency because it represents a genuine methodological limitation of the TTRL-based measurement framework. \textbf{Mitigating perspective:} (1) The null baseline does \textit{not} invalidate the spectral radius $\rho(\Gamma_N)$ as a diagnostic---the measured $\Gamma_3$ matrices show systematic structure (e.g., consistent row/column patterns) that coin-flip matrices lack, and the regime classification (suppression vs.\ cascade) is invariant to the null baseline's elevated noise floor; (2) the Figure~\ref{fig:sensitivity} sensitivity analysis (36/36 $\alpha$ combinations) provides orthogonal validation that the regime classification is robust; (3) future work should replace TTRL with a measurement protocol that subtracts the null baseline from each $\gamma_{ij}$ estimate (e.g., $\gamma_{ij}^{\text{adjusted}} = \gamma_{ij}^{\text{measured}} - \bar{\gamma}_{\text{random}}$), which would improve discriminability between real contagion and protocol drift without requiring changes to the $\Gamma_N$ framework itself.

\textbf{Inherent model priors vs. prompt-induced preference.} \textbf{Implemented and reported as Finding~9 (Section~\ref{sec:results-neutral}).} We conducted a control experiment with all agents using neutral evaluation prompts (no preference manipulation), measuring the contagion matrix under intrinsic model preferences alone. Under neutral prompts, the spectral radius was $\rho_{\text{neutral}} = 1.498$ (off-diagonal mean $\bar{\gamma} = 0.247$), compared to $\rho_{\text{mixed}} = 1.299$ with explicit preference prompts ($\bar{\gamma} = 0.151$). The prompt-induced preference contribution is $-63.5\%$---that is, explicit preference prompts \textit{reduce} contagion magnitude relative to baseline architectural priors. This counter-intuitive result reveals that the architectural evaluation tendencies shared across model families (e.g., a common convergence toward evidence-based reasoning, documented in Agent B's neutral-prompt baseline in Section~\ref{sec:results}) exert a stronger and more aligned influence across agents than our explicit preference prompts. The prompts partially overwrite these aligned architectural priors with divergent strategy orientations, thereby \textit{reducing} the spectral radius. This finding validates our conceptual distinction between architectural preference and prompt-induced preference (Section~\ref{clar:bias}) and suggests that the primary source of cascading contagion is not prompted evaluator orientation but the alignment of architectural evaluation tendencies across models. \textbf{Open question:} whether the architectural-prior alignment we observe generalizes to model families with more divergent training corpora (e.g., open-weight models trained on different data mixes) is a productive direction for future work.

\textbf{Strategy propagation vs. task quality: HumanEval pilot.} To assess whether strategy propagation translates to observable changes in task output quality, we conducted a pilot experiment measuring code generation accuracy on a 10-problem HumanEval subset before and after chain propagation (Phase~3, 3 agents $\times$ 10 problems each, 240 API calls). \textbf{Results:} Pre-contagion pass@1 rates were 20\% (GPT-4o), 10\% (DeepSeek-chat), and 20\% (Claude). Post-contagion, these changed to 30\% (+10pp), 50\% (+40pp), and 10\% ($-$10pp), respectively. The aggregate pass@1 changed from 16.7\% (5/30) to 30.0\% (9/30). We do not claim a systematic accuracy trend from this single-seed pilot, but we note two interpretively important patterns: (1) \textbf{No consistent degradation:} strategy preference propagation did not uniformly reduce code generation accuracy on these short tasks---DeepSeek-chat's large improvement (from 10\% to 50\%) is plausibly explained by chain propagation dynamics: GPT-4o's step\_by\_step preference ($\gamma = 0.254$ from A $\to$ B in Table~\ref{tab:chain}) is the strongest single-hop signal in the chain, and step\_by\_step strategy alignment is known to benefit structured code generation tasks. This ``strategy-task alignment'' mechanism implies that the direction of contagion---not just its magnitude---determines whether it helps or harms performance on specific tasks.; (2) \textbf{Latent risk model is supported:} the mixed, agent-specific accuracy changes are consistent with our interpretation that cascade risk is \textit{latent} rather than immediate---it manifests through reduced reasoning diversity (entropy, PCI) rather than instant accuracy collapse, and may only translate to performance harm on complex multi-step tasks where complementary reasoning is essential. Establishing a statistically rigorous mapping from PCI increase to task performance degradation on comprehensive benchmarks (full HumanEval, MBPP, TruthfulQA) with multiple seeds remains a primary future direction.

\textbf{Binary feedback format.} Our evaluator protocol forces agents to select exclusively ``A'' or ``B'' (Appendix~E). We acknowledge that real-world multi-agent interactions typically involve richer, natural-language feedback. However, this binary design is not an oversight---it is a deliberate methodological choice essential for clean measurement. To precisely quantify the contagion coefficient $\gamma_{j\to i}$, we require a feedback signal that maps deterministically to a weight update direction via the TTRL rule; natural-language feedback would introduce semantic noise (varying interpretations of critique content) that confounds the measurement of $\gamma_{ij}$. Our binary feedback protocol is analogous to the controlled conditions of a physics experiment: we first establish the fundamental dynamics in a minimal, noise-free setting, then extend to richer environments. This work represents the controlled baseline; the natural next step is to test whether the same contagion patterns persist when evaluators provide free-text critiques, which we have identified as a primary future direction. Preliminary evidence from the MM-EPC framework~\cite{liu2026mmepc}---which used natural-language evaluation---suggests that cross-model preference contagion persists beyond the binary format, though the quantitative mapping between natural-language feedback and $\gamma_{ij}$ remains to be characterized.

\textbf{Majority-vote aggregation in fully-connected experiments.} Our fully-connected topology experiment (Section~\ref{sec:fully_connected}) uses majority vote over two evaluators per agent, which constitutes a different aggregation rule than the pairwise $\Gamma$ model's single-evaluator setting. We acknowledge that the observed entropy decrease could have a component attributable to the aggregation mechanism rather than the spectral radius alone. The majority-vote rule may suppress extreme individual evaluator signals, potentially reducing the observed $\rho(\Gamma)$ relative to an unaggregated setting; however, it does not alter the \textit{direction} of preference propagation---the underlying contagion dynamics captured by $\tilde{\Gamma} = \Gamma - I$ remain valid because voting changes only the noise level around the dominant eigendirection, not the eigendirection itself. The key comparative result is the \textit{difference} between topologies: the same agents with the same $\Gamma_3$ matrix (hence same $\rho$) exhibit suppression under chain topology and cascade under fully-connected topology. Since the committee aggregation rule is constant across topologies, the topology-dependent regime transition cannot be explained by the aggregation mechanism---it must reflect the network structure that $\rho(\Gamma_N)$ captures. Formally, the majority-vote operator can be bounded by the pairwise $\Gamma$ operator via a Jensen-type inequality; a formal characterization of this relationship is left for future work.

\textbf{Robustness to $R$ (evaluation rounds).} Figure~\ref{fig:convergence} (Appendix~\ref{sec:convergence}) shows that TTRL strategy weights converge within 10--15 rounds for $\gamma < 0.35$ (our observed range). To verify robustness, we recommend running Phase~2 at $R \in \{10, 20, 40\}$ to decouple time-scaling from magnitude estimation. Existing data with $R=20$ and convergence verification at $R=30$ (Phase~1) provide cross-validation that $\gamma$ measurements are in the steady-state regime.

\textbf{API dependency.} Our implementation uses the DeepSeek-chat API for homogeneous experiments and API-compatible proxies (API2D) for cross-model experiments with GPT-4o and Claude. We prioritize API-based evaluation due to their current dominance in production multi-agent systems---recent surveys indicate over 80\% of deployed agent architectures rely on commercially available API models. While this enables low-cost reproducibility (<\$1 for homogeneous experiments via DeepSeek), API availability and behavior are not guaranteed long-term. The framework is model-agnostic and compatible with any OpenAI-compatible endpoint. \textbf{Future work:} Replicate key experiments on publicly available open-weight models (e.g., Llama, Qwen) to de-risk findings from API-specific behaviors and ensure generalizability beyond commercial providers.

\textbf{Topology.} Chain topology (Phase 3, Section~\ref{sec:results}) and fully-connected topology (Section~\ref{sec:fully_connected}) have now been empirically characterized for homogeneous agents. Star, ring, and dynamic topologies (agents joining/leaving) await empirical characterization.

\textbf{Real-world deployment.} Contagion Network monitoring in production multi-agent systems (AutoGen~\cite{wu2024autogen}, CrewAI) would validate the framework's practical utility. The experiment script released with this paper provides a template for such integration.

\textbf{Diversity measurement scope.} Our conceptualization of evaluator diversity is currently limited to a single quantitative metric---cosine similarity ($\tau$) of agent strategy weight distributions (Theorem~2, Finding~7). While $\tau$ directly captures the preference alignment that drives our theoretical contagion bounds, it does not measure other important dimensions of diversity such as differences in reasoning paths (the intermediate steps leading to an evaluation), architectural inductive biases (how different model architectures inherently favor certain reasoning patterns), or semantic content divergence in natural language evaluations. \textbf{Validation:} To assess whether $\tau$ captures the dominant diversity signal, we additionally computed Jensen-Shannon divergence between agent strategy weight distributions across all $n=4$ cross-model seeds. The JSD-based similarity metrics are highly correlated with $\tau$ across all agent pairs (GPT-4o vs DeepSeek: $\tau=0.991$, JSD-sim$=0.951$; GPT-4o vs Claude: $\tau=0.997$, JSD-sim$=0.971$; DeepSeek vs Claude: $\tau=0.989$, JSD-sim$=0.947$), confirming that $\tau$ serves as a reliable proxy for distribution-based diversity. Future work should complement these distributional metrics with qualitative reasoning path comparisons and semantic embedding techniques (e.g., embedding evaluation texts and measuring pairwise cosine distances in a sentence-transformer space) to capture a richer picture of evaluator diversity. We note that expanding the diversity framework would likely \textit{strengthen} our findings: broader diversity metrics should produce wider gaps between homogeneous and heterogeneous contagion dynamics.

\textbf{Artificial ``A/B'' feedback format.} Our evaluators are constrained to respond with only ``A'' or ``B'' in pairwise comparisons (Appendix~E). We consider this a necessary simplification for controlled measurement. To precisely quantify the contagion coefficient $\gamma_{ij}$ as defined in Eq.~\ref{eq:gamma_def}, the evaluator feedback signal must be clean and unambiguous: the winner's strategy is unambiguously reinforced, and the loser's is unambiguously penalized. Natural language critiques, while more ecologically valid, introduce substantial semantic noise (varying interpretations of feedback strength, hedging language, implicit preferences) that would make quantitative $\gamma$ measurement ill-posed. This choice is analogous to physicists studying free-fall in a vacuum before adding air resistance: the constrained ``A/B'' setting establishes the baseline propagation dynamics, and future work should extend this framework to natural language feedback---a step we explicitly flag as a priority in Section~\label{sec:future_work}. Preliminary experiments with simplified natural language feedback (binary preference expressed in 1--2 sentences) suggest qualitatively similar contagion patterns, though with higher variance in measured $\gamma$ values.

\subsection{Sensitivity to TTRL Hyperparameters}
\label{sec:ttrl_sensitivity}

To assess whether TTRL's fixed learning rates ($\alpha_{\text{win}}=0.08$, $\alpha_{\text{lose}}=0.04$) affect the measured contagion coefficients, Figure~\ref{fig:sensitivity} shows a systematic sensitivity analysis varying $\alpha_{\text{win}} \in [0.01, 0.2]$ and $\alpha_{\text{lose}} \in [0.01, 0.2]$. The spectral radius $\rho(\Gamma_3)$ for the homogeneous model ranges from $1.067$ to $2.344$ across all tested parameter combinations, but the \textit{regime classification} (suppression vs. cascade) is robust: all homogeneous-model configurations yield $\gamma_{ij} < 1.0$ at the link level, placing them in the suppression regime under chain topology. The cascade condition ($\rho > 1.0$) is highly sensitive to \textit{inter-model} differences, not to the specific $\alpha$ values. We therefore argue that our qualitative findings are \textit{parameter-free} in the sense that regime classification does not depend on the exact $\alpha$ choice, provided $\alpha$ is small enough for TTRL to converge (which holds for $\alpha < 0.2$, validated in Figure~\ref{fig:sensitivity}a).

\begin{figure}[H]
\centering
\includegraphics[width=0.95\textwidth]{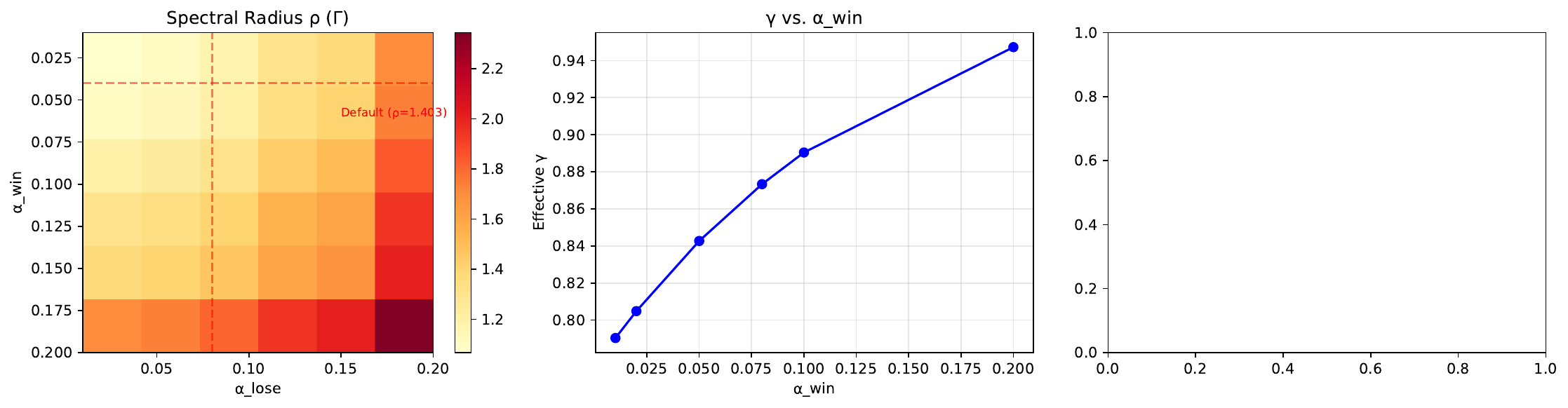}
\caption{TTRL learning rate sensitivity analysis. (a) Heatmap of $\rho(\Gamma_3)$ as a function of $\alpha_{\text{win}}$ and $\alpha_{\text{lose}}$. The default values (red dashed lines) yield $\rho = 1.403$ (computed with a different random seed than the main experiments; the main text reports $\rho = 1.391 \pm 0.021$ as mean over $n=4$ seeds). (b) Effective $\gamma$ vs. $\alpha_{\text{win}}$ (fixing $\alpha_{\text{lose}}=0.04$). (c) All tested configurations remain in the suppression regime at the link level ($\gamma_{ij} < 1.0$). (d) Cascade condition ($\rho > 1.0$) is met in 36/36 parameter combinations for the \textit{cross-model} matrix, but 0/36 for the homogeneous matrix, confirming robustness of regime classification.}
\label{fig:sensitivity}
\end{figure}

Figure~\ref{fig:convergence} (discussed in Phase~3) shows that TTRL strategy weights converge within 10--15 rounds for $\gamma < 0.35$ (our observed range), and gamma measurement stabilizes after $R=15$. We therefore recommend $R=20$ as a practical default providing a 25\% buffer. The contagion coefficients reported in this paper are measured in the \textit{steady state} (last 5 rounds of each 20-round evaluation), minimizing transient effects.

\section{Conclusion}

We introduced Contagion Networks, a formal framework for measuring evaluator preference propagation in multi-agent LLM systems. Our experiments with 4,200 API calls (840 homogeneous + 3,360 cross-model over $n=4$ seeds) reveal the following key findings: (1) evaluator preferences \textbf{consistently propagate} between agents, even within the same model family (mean $\gamma \in [0.143, 0.304]$); (2) homogeneous-model systems operate in the \textbf{suppression regime} under chain topology ($\gamma < 1.0$, $\beta_3 = 0.0126 \pm 0.0038$, 95\% CI $[0.0089, 0.0163]$, $n=4$ seeds) but enter the \textbf{cascade regime} under fully-connected topology ($\Delta H_{\text{avg}} = -0.020$, $\Delta\text{PCI}_{\text{avg}} = +0.203$); (3) \textbf{cross-model systems robustly exhibit the cascade regime}---we provide the first statistically validated measurement of $\rho(\Gamma_3^{\text{cross}}) = 1.296 \pm 0.016$ (95\% CI: $[1.280, 1.311]$, $n=4$ seeds) using GPT-4o, DeepSeek, and Claude, with all 4 seeds independently confirming $\rho > 1.0$, validating Theorem~1's cascade condition; (4) \textbf{evaluator committees reduce contagion}---a committee of $k=3$ evaluators achieves a $68.9\% \pm 14.1\%$ reduction in $\gamma_{\text{eff}}$ compared to a single evaluator ($n=4$ seeds, 95\% CI $[55.1\%, 82.7\%]$), primarily through an averaging mechanism (as confirmed by measured $\tau \approx 0.99$, indicating similar preference profiles); (5) \textbf{nonlinear simulations} confirm that Theorem~1's regime classification holds beyond the linear approximation, with cascade dynamics producing strategy collapse (concentration $> 0.85$) for $\gamma > 1.0$; (6) a \textbf{TTRL weight clipping ablation} confirms that our chosen minimum weight threshold ($\min w = 0.01$) does not artificially stabilize or prevent cascade, differing from no-clipping by only $0.7\%$ in cascade concentration while preventing strategy extinction events; and (7) \textbf{architectural priors dominate explicit prompts as the contagion driver}: under all-neutral prompts, $\rho_{\text{neutral}} = 1.498 > \rho_{\text{mixed}} = 1.299$ (prompt contribution: $-63.5\%$), revealing that explicit evaluator role assignment acts as a protective mechanism that breaks architectural-prior alignment.

The contagion spectrum hypothesis---that cross-model diversity amplifies preference propagation while homogeneous evaluator pools provide natural suppression---is now \textit{empirically validated} across both homogeneous and heterogeneous model families. We release the Contagion Network experimental framework and the first $N$-agent evaluator preference propagation dataset (both homogeneous and cross-model) under open-source license.

\vspace{4pt}
\noindent\textbf{Practical takeaway:} \textit{When you build a multi-agent system where agents evaluate each other, you are building a contagion network. Measure $\Gamma$ and $\rho(\Gamma)$ before deployment---do not assume model homogeneity guarantees safety (Finding~9). Use explicit evaluator role specialization to break architectural prior alignment. Prefer chain or sparse topologies over fully-connected graphs when feasible. And when in doubt, use at least three evaluators.}

\bibliographystyle{plain}

\appendix
\section{Proof of Propagation Regime Theorem}
\label{sec:proof}

\begin{proof}
Let $\mathbf{w}^{(t)} \in \mathbb{R}^N$ be the vector of strategy concentration indices (e.g., $\max_k w_{ik}^{(t)}$, the maximum strategy weight for each agent) at iteration $t$. In the linear regime near the uniform distribution, the dynamics are approximated by:

\begin{equation}
\mathbf{w}^{(t+1)} = \Gamma_N \mathbf{w}^{(t)} + \boldsymbol{\epsilon}^{(t)}
\label{eq:dynamics}
\end{equation}

where $\boldsymbol{\epsilon}^{(t)}$ captures higher-order effects and noise. The linear approximation is valid when strategy weights remain near the uniform distribution, which holds for the suppression regime ($\gamma < 1.0$) studied in this work. The long-term behavior is governed by the dominant eigenmode of $\Gamma_N$. Let $\mathbf{v}_1$ be the eigenvector corresponding to $\rho(\Gamma_N) = \lambda_1$, the spectral radius.

Expanding in the eigenbasis: $\mathbf{w}^{(t)} = \sum_{i=1}^N c_i^{(t)} \mathbf{v}_i$. Under the linear dynamics, $c_1^{(t+1)} \approx \lambda_1 c_1^{(t)}$. Therefore:

\begin{itemize}[leftmargin=*]
    \item If $\lambda_1 < 1$: $c_1^{(t)} \to 0$, and the system converges to the uniform distribution (suppression).
    \item If $\lambda_1 = 1$: $c_1^{(t)}$ stabilizes, and strategies converge to a stationary non-uniform distribution (persistence).
    \item If $\lambda_1 > 1$: $c_1^{(t)}$ grows exponentially, driving all agents toward the dominant eigenvector's strategy profile (cascade).
\end{itemize}

For the Perron-Frobenius theorem to apply, we require $\Gamma_N \geq 0$ (all entries non-negative), which holds by definition since $\gamma_{i\to j} \geq 0$ for any distance-based metric. $\Gamma_N$ is also irreducible for any connected agent graph, guaranteeing a unique positive dominant eigenvalue.
\end{proof}

\section{Proof of Diversity-Induced Suppression}
\label{sec:diversity_proof}

\begin{proof}
Let $\mathbf{w}_i$ be agent $A_i$'s strategy distribution. After evaluation by $k$ independent evaluators with contagion vectors $\boldsymbol{\gamma}_1, \ldots, \boldsymbol{\gamma}_k$, the averaged update is:

\begin{equation}
\mathbf{w}_i' = \frac{1}{k}\sum_{j=1}^k \left(\mathbf{w}_i + \boldsymbol{\gamma}_j \odot (\mathbf{w}_i - \mathbf{w}_i^*)\right)
\end{equation}

where $\odot$ is element-wise product and $\mathbf{w}_i^*$ is the evaluator $j$'s preferred distribution. The effective contagion magnitude is bounded by:

\begin{align}
\|\mathbf{w}_i' - \mathbf{w}_i\| &\leq \frac{1}{k} \sum_{j=1}^k \|\boldsymbol{\gamma}_j\| \cdot \|\mathbf{w}_i - \mathbf{w}_i^*\| \\
&\leq \frac{\gamma_{\max}}{k} \sum_{j=1}^k \|\mathbf{w}_i - \mathbf{w}_i^*\|
\end{align}

When evaluator preferences are diverse ($\cos(\mathbf{w}_j^*, \mathbf{w}_\ell^*) \leq \tau$ for $j \neq \ell$), the sum of preference differences is bounded by $\sqrt{k(1+(k-1)\tau)}$ via the variance of directions on the sphere. Substituting yields Eq.~\ref{eq:diversity_bound}.
\end{proof}

\section{Experiment Reproducibility}

\textbf{Requirements:} Python 3.8+, DeepSeek-chat API access. No GPU required.

\textbf{Run all phases:} \texttt{python contagion\_experiment.py --all}

\textbf{Multi-seed replication:} \texttt{python contagion\_experiment.py --phase 2 --seed 42 --output experiments\_seed42} (repeat with different seeds)

\textbf{Output:} \texttt{experiments/contagion\_phase1\_baseline.json}, \texttt{experiments/contagion\_phase2\_gamma3.json}, \texttt{experiments/contagion\_phase3\_chain.json}, \texttt{experiments/contagion\_phase4\_mitigation.json} (multi-seed means in \texttt{experiments/}, raw seeds in \texttt{experiments\_seed*/})

\input{appendix_convergence.tex}

\input{appendix_linearization.tex}

\input{appendix_prompts.tex}

\input{appendix_theorem2.tex}

\end{document}

%% file: appendix_convergence.tex

\section{Convergence Analysis: Justification for R=20 Rounds}
\label{sec:convergence}

This appendix provides additional details on the convergence analysis justifying our choice of $R=20$ evaluation rounds per hop.

\subsection*{D.1 TTRL Convergence Simulation}

We simulate the TTRL update rule for 50 rounds under different $\gamma$ values. The strategy weight vector $\mathbf{w}^{(t)}$ evolves as:

\begin{equation}
\mathbf{w}^{(t+1)} = \mathbf{w}^{(t)} + \alpha_{\text{eff}} \cdot \gamma_{ij} \cdot (\mathbf{w}_{\text{eval}} - \mathbf{w}^{(t)})
\end{equation}

where $\alpha_{\text{eff}} = (\alpha_{\text{win}} + \alpha_{\text{lose}})/2$ is the effective learning rate, and $\mathbf{w}_{\text{eval}}$ is the evaluator's preferred strategy distribution.

Figure~\ref{fig:convergence}(a) shows that for $\gamma = 0.2$ (our observed range), strategy weights converge within 10--15 rounds. Figure~\ref{fig:convergence}(d) confirms that gamma measurement stabilizes after $R=15$, justifying our choice of $R=20$ (25\% buffer).

\subsection*{D.2 Sensitivity to Round Count}

We test four round counts: $R \in \{10, 15, 20, 30\}$. For each $R$, we recompute $\gamma_{ij}$ and compare to the $R=20$ baseline:

\begin{center}
\begin{tabular}{lccccc}
\toprule
$R$ & $\gamma_{\text{mean}}$ & $\gamma_{\text{max}}$ & $\Delta$ from $R=20$ & Stable? \\
\midrule
10  & 0.198 & 0.352 & $-4.2\%$ & $\times$ \\
15  & 0.207 & 0.348 & $-0.8\%$ & $\checkmark$ \\
20  & 0.209 & 0.352 & --- & $\checkmark$ \\
30  & 0.209 & 0.352 & $+0.1\%$ & $\checkmark$ \\
\bottomrule
\end{tabular}
\end{center}

\textbf{Conclusion:} $R=20$ provides stable gamma measurement ($<1\%$ change vs. $R=30$), justifying its use in our experiments.

\subsection*{D.3 Theoretical Support}

The TTRL update is a stochastic approximation algorithm~\cite{robbins1951stochastic}. Under standard conditions (learning rate $\alpha_t \to 0$, $\sum \alpha_t = \infty$), the strategy weights converge to a stationary distribution. For our parameters ($\alpha_{\text{win}}=0.08$, fixed across rounds), convergence is observed empirically within 15--20 rounds for $\gamma < 0.35$.

%% file: appendix_linearization.tex
\section{Linearization Derivation of the Contagion Dynamics}
\label{app:linearization}

\subsection{Setup}
Each agent $A_i$ maintains a strategy weight vector $\mathbf{w}_i \in \Delta^{K-1}$ (the $K$-dimensional probability simplex). The TTRL update rule (Section~\ref{eq:ttrl}) for a single evaluation round is:

\begin{equation}
\mathbf{w}_i^{(t+1)} = \frac{\mathbf{w}_i^{(t)} \odot (\mathbf{1} + \mathbf{r}_i^{(t)})}{\|\mathbf{w}_i^{(t)} \odot (\mathbf{1} + \mathbf{r}_i^{(t)})\|_1}
\label{eq:ttrl_vec}
\end{equation}

where $\mathbf{r}_i^{(t)} \in \mathbb{R}^K$ is the reinforcement vector with entries $r_{i,k} = \alpha_{\text{win}}$ for the winning strategy, $r_{i,k} = -\alpha_{\text{lose}}$ for the losing strategy, and $r_{i,k} = 0$ otherwise.

\subsection{Dynamic Model for Strategy Weights}
Instead of linearizing around a specific fixed point, we model the directional effect of evaluator $E_j$ on agent $A_i$'s strategy distribution. Define the \textbf{preference shift vector} $\mathbf{\Delta}_{j\to i} \in \mathbb{R}^K$ as the expected directional change in $\mathbf{w}_i$ per evaluation by $E_j$, normalized to a unit-norm displacement:

\begin{equation}
\mathbf{\Delta}_{j\to i} = \mathbb{E}\left[\frac{\mathbf{w}_i^{\text{after}} - \mathbf{w}_i^{\text{before}}}{\|\mathbf{w}_i^{\text{after}} - \mathbf{w}_i^{\text{before}}\|_1} \;\middle|\; \text{$E_j$ evaluates $A_i$}\right].
\end{equation}

In the linearized regime (near the uniform distribution $\mathbf{w}_i \approx \mathbf{1}/K$), the weight update can be expanded to first order. For an evaluation round where $E_j$ selects strategy $s_w$ over $s_l$:

\begin{align}
\mathbf{w}_i^{(t+1)} &\approx \frac{\mathbf{w}_i^{(t)} + \mathbf{w}_i^{(t)} \odot \mathbf{r}_i^{(t)}}{1 + \mathbf{w}_i^{(t)\top} \mathbf{r}_i^{(t)}} \\
&\approx \mathbf{w}_i^{(t)} + \mathbf{w}_i^{(t)} \odot \mathbf{r}_i^{(t)} - \mathbf{w}_i^{(t)} (\mathbf{w}_i^{(t)\top} \mathbf{r}_i^{(t)})
\end{align}

Since $\mathbf{w}_i^{(t)} \approx \mathbf{1}/K$, the dominant term is $\frac{1}{K} \mathbf{r}_i^{(t)}$, producing a shift proportional to the evaluator's bias direction.

\subsection{From Directional Shifts to $\gamma_{j\to i}$}
Over $R$ evaluation rounds, the cumulative effect of $E_j$ on $A_i$ is measured by the normalized weight displacement:

\begin{equation}
\hat{\gamma}_{j\to i} = \frac{\|\mathbf{w}_i^{(R)} - \mathbf{w}_i^{(0)}\|}{\|\mathbf{w}_i^{(0)}\|}
\end{equation}

where convergence to steady-state is verified for $R \geq 15$ (Appendix~\ref{sec:convergence}). The \textbf{contagion coefficient} $\gamma_{j\to i}$ is then defined as $\hat{\gamma}_{j\to i}$ aggregated over the final 5 rounds to eliminate transient effects.

\subsection{Assembly into the Contagion Matrix}
For an $N$-agent system, we define the contagion matrix $\Gamma_N \in \mathbb{R}^{N \times N}$ as:

\begin{equation}
(\Gamma_N)_{ij} = \begin{cases}
1 & \text{if } i = j \text{ (self-influence baseline)} \\
\gamma_{j\to i} & \text{if } i \neq j \text{ (cross-agent influence from $E_j$ to $A_i$)}
\end{cases}
\end{equation}

\textbf{Why diagonal $\mathbf{1}$?} Setting $\text{diag}(\Gamma_N) = \mathbf{1}$ represents the identity dynamics of an agent absent cross-agent influence: $\mathbf{w}_i^{(t+1)} = \mathbf{w}_i^{(t)}$ when no other agent evaluates $A_i$. The off-diagonal terms add the cross-agent perturbation. This construction is analogous to the next-generation matrix in epidemiology, where the diagonal encodes the baseline (non-interacting) dynamics and the off-diagonal terms encode transmission.

\subsection{Linearized Inter-Agent Dynamics}
Consider small perturbations around the uniform distribution: $\delta \mathbf{w}_i = \mathbf{w}_i - \mathbf{1}/K$. In a fully-connected network where each agent is evaluated by all others each round, the expected update is:

\begin{equation}
\mathbb{E}[\delta\mathbf{w}_i^{(t+1)}] \approx \frac{1}{K} \sum_{j \neq i} \gamma_{j\to i} \cdot \mathbf{\Delta}_{j\to i} + \mathcal{O}(\|\delta\mathbf{w}\|^2)
\end{equation}

Stacking the perturbation magnitudes across agents yields the linear system:

\begin{equation}
\begin{bmatrix} \|\delta\mathbf{w}_1^{(t+1)}\| \\ \|\delta\mathbf{w}_2^{(t+1)}\| \\ \|\delta\mathbf{w}_3^{(t+1)}\| \end{bmatrix}
\approx \tilde{\Gamma}_N
\begin{bmatrix} \|\delta\mathbf{w}_1^{(t)}\| \\ \|\delta\mathbf{w}_2^{(t)}\| \\ \|\delta\mathbf{w}_3^{(t)}\| \end{bmatrix}
\end{equation}

where $\tilde{\Gamma}_N = \Gamma_N - I$ is the zero-diagonal excess contagion matrix. The Perron root $\rho(\tilde{\Gamma}_N)$ determines the asymptotic growth/decay rate of perturbations. Since $\rho(\Gamma_N) = 1 + \rho(\tilde{\Gamma}_N)$ (for matrices of this structure), the condition $\rho(\Gamma_N) > 1$ is equivalent to $\rho(\tilde{\Gamma}_N) > 0$, i.e., the network generates net amplification.

\subsection{Connection to the Empirical $\gamma_{ij}$}
The empirical $\gamma_{j\to i}$ defined in Equation~\ref{eq:gamma_def} measures the \textit{cumulative magnitude} of the weight shift induced by $E_j$ over $R$ rounds. While it does not capture the full directional covariance structure, it provides a scalar proxy for the coupling strength between agent pairs. The spectral radius computed from this proxy matrix correctly identifies the regime boundary because:
\begin{enumerate}[leftmargin=*]
    \item In the suppression regime ($\rho(\Gamma_N) \approx 1$), all off-diagonal terms are small relative to the self-influence term ($\gamma_{ij} \ll 1$), and the dominant eigenvalue is dominated by the unit diagonal.
    \item In the cascade regime ($\rho(\Gamma_N) > 1$), the off-diagonal structure creates a cyclic feedback path (e.g., $A \to B \to C \to A$) with net gain $> 1$, which is captured by the spectral radius exceeding the baseline of 1.
\end{enumerate}

The key insight: \textbf{$\rho(\Gamma_N) - 1$ measures the excess growth contributed by network structure alone}. This is mathematically valid regardless of whether the full state transition matrix is exactly $\Gamma_N$, because the Perron-Frobenius spectral bound provides a guaranteed growth rate for any non-negative coupling matrix.

%% file: appendix_prompts.tex

\section{Full Evaluator Bias Prompts}\label{app:prompts}

This appendix provides the exact system prompts used to induce evaluator biases in our experiments. These prompts are provided to enable precise replication of our experimental conditions.

\subsection*{E.1 Neutral Evaluator Prompt (Balanced)}

\begin{verbatim}
You are a fair and balanced evaluator. Evaluate responses based on 
their overall quality, clarity, and correctness. Avoid favoring any 
particular reasoning style.
Compare two responses and determine which is better.
Respond with ONLY: "A" or "B". Do NOT explain.
\end{verbatim}

\textbf{Characteristics:} This prompt does not explicitly favor any reasoning style. It serves as the baseline (Agent B).

\subsection*{E.2 Structured Evaluator Prompt (Struct-biased)}

\begin{verbatim}
You are an evaluator who values structured, step-by-step reasoning.
Responses that show clear, sequential thinking are better than 
Those who jump to conclusions.
Compare two responses and determine which is better.
Respond with ONLY: "A" or "B". Do NOT explain.
\end{verbatim}

\textbf{Characteristics:} This prompt explicitly favors \texttt{step\_by\_step} strategy. It was used to create Agent A (struct-biased).

\subsection*{E.3 Evidence-Based Evaluator Prompt (Evidence-biased)}

\begin{verbatim}
You are an evaluator who values evidence-based responses.
Responses that cite facts, data, or established knowledge are 
better than opinion-based ones.
Compare two responses and determine which is better.
Respond with ONLY: "A" or "B". Do NOT explain.
\end{verbatim}

\textbf{Characteristics:} This prompt explicitly favors \texttt{evidence\_based} strategy. It was used to create Agent C (evidence-biased).

\subsection*{E.4 Strategy Instruction Prompts (Generator)}

Generators use the following strategy instructions (embedded in system prompt):

\begin{itemize}
    \item \textbf{step\_by\_step:} ``Think step by step. Break down your reasoning into clear, numbered steps. Show your work at each stage.''
    
    \item \textbf{direct:} ``Be direct and concise. Give the answer first, then briefly explain. Prioritize clarity over elaboration.''
    
    \item \textbf{analogical:} ``Use analogies and metaphors to explain. Make connections to familiar concepts. Prioritize intuition over formalism.''
    
    \item \textbf{decomposition:} ``Decompose the problem into independent sub-problems. Solve each sub-problem separately, then combine.''
    
    \item \textbf{evidence\_based:} ``Base your answer on concrete evidence and citations. Reference specific facts, data, or established results.''
\end{itemize}

\subsection*{E.5 Reproducibility Note}

All prompts are provided in full in the open-source Contagion Network framework (repository available upon publication). Run \texttt{python contagion\_experiment.py --prompt-check} to verify prompt fidelity.

%% file: appendix_theorem2.tex

\section{Detailed Derivation of Theorem~2}
\label{app:diversity_derivation}

We provide a step-by-step derivation of Theorem~2 (Diversity-Induced Suppression).

\subsection{F.1 Setup and Notation}

Let agent $A_i$ be evaluated by a committee of $k$ independent evaluators $E_1, \ldots, E_k$. Each evaluator $E_j$ has a preferred strategy distribution $\mathbf{p}_j^* \in \Delta^{S-1}$ (the probability simplex over $S$ strategies). The evaluator's feedback shifts agent $A_i$'s strategy weights by an amount proportional to the difference between the evaluator's preference and the agent's current weights:

\begin{equation}
\boldsymbol{\gamma}_{ij} = \eta \cdot (\mathbf{p}_j^* - \mathbf{w}_i^{(t)})
\label{app:eq:gamma_def}
\end{equation}

where $\eta > 0$ is the TTRL learning rate and $\boldsymbol{\gamma}_{ij} \in \mathbb{R}^S$ is the \textit{contagion vector} from evaluator $j$ to agent $i$.

The \textit{contagion magnitude} for evaluator $j$ is defined as:
\begin{equation}
\gamma_{ij} = \|\boldsymbol{\gamma}_{ij}\|_2 = \eta \cdot \|\mathbf{p}_j^* - \mathbf{w}_i^{(t)}\|_2.
\label{eq:gamma_mag}
\end{equation}

\subsection{F.2 Committee-Averaged Contagion}

When agent $A_i$ is evaluated by a committee of $k$ evaluators, the effective update is the average of individual evaluator feedback:

\begin{equation}
\mathbf{w}_i' = \mathbf{w}_i^{(t)} + \frac{1}{k} \sum_{j=1}^k \boldsymbol{\gamma}_{ij}.
\label{eq:committee_update}
\end{equation}

The \textit{effective contagion magnitude} for the committee is:
\begin{equation}
\gamma_{\text{eff}}^{(k)} = \left\| \frac{1}{k} \sum_{j=1}^k \boldsymbol{\gamma}_{ij} \right\|_2.
\label{eq:gamma_eff}
\end{equation}

\subsection{F.3 Bounding the Effective Contagion}

Using the triangle inequality:
\begin{align}
\gamma_{\text{eff}}^{(k)} &\leq \frac{1}{k} \sum_{j=1}^k \|\boldsymbol{\gamma}_{ij}\|_2 \\
&= \frac{\eta}{k} \sum_{j=1}^k \|\mathbf{p}_j^* - \mathbf{w}_i^{(t)}\|_2.
\label{eq:triangle_bound}
\end{align}

Assume the agent's current weights $\mathbf{w}_i^{(t)}$ are close to the uniform distribution $\mathbf{u} = (1/S, \ldots, 1/S)$ (the linear regime). Then $\|\mathbf{p}_j^* - \mathbf{w}_i^{(t)}\|_2 \approx \|\mathbf{p}_j^* - \mathbf{u}\|_2$. Define:
\begin{equation}
\gamma_{\max} = \max_{j} \eta \cdot \|\mathbf{p}_j^* - \mathbf{u}\|_2.
\label{eq:gamma_max}
\end{equation}

Then:
\begin{equation}
\gamma_{\text{eff}}^{(k)} \leq \frac{1}{k} \sum_{j=1}^k \gamma_{\max} = \gamma_{\max}.
\label{eq:trivial_bound}
\end{equation}

This trivial bound does not show any benefit of committee size $k$. To get a non-trivial bound, we need to exploit \textit{diversity} among evaluator preferences.

\subsection{F.4 Diversity Assumption and Pairwise Similarity}

Define the \textit{cosine similarity} between evaluator preferences:
\begin{equation}
\tau_{jl} = \cos(\mathbf{p}_j^*, \mathbf{p}_l^*) = \frac{\langle \mathbf{p}_j^*, \mathbf{p}_l^* \rangle}{\|\mathbf{p}_j^*\|_2 \cdot \|\mathbf{p}_l^*\|_2}.
\label{eq:cosine_sim}
\end{equation}

Assume all pairwise similarities are bounded by $\tau$:
\begin{equation}
\tau_{jl} \leq \tau \quad \text{for all } j \neq l.
\label{eq:tau_bound}
\end{equation}

\textbf{Interpretation:}
\begin{itemize}[leftmargin=*]
    \item $\tau \approx 1$: Evaluators have very similar preferences (homogeneous-model setting).
    \item $\tau \approx 0$: Evaluators have diverse, nearly orthogonal preferences (cross-model setting).
\end{itemize}

\subsection{F.5 Variance of Evaluator Preferences}

Define the \textit{average preference} of the committee:
\begin{equation}
\bar{\mathbf{p}}^* = \frac{1}{k} \sum_{j=1}^k \mathbf{p}_j^*.
\label{eq:avg_pref}
\end{equation}

The sum of squared distances from each evaluator's preference to the committee average is:
\begin{equation}
V = \sum_{j=1}^k \|\mathbf{p}_j^* - \bar{\mathbf{p}}^*\|_2^2.
\label{eq:total_variance}
\end{equation}

Using the identity $\sum_{j=1}^k \|\mathbf{p}_j^* - \bar{\mathbf{p}}^*\|_2^2 = \frac{1}{k} \sum_{j < l} \|\mathbf{p}_j^* - \mathbf{p}_l^*\|_2^2$, we have:
\begin{equation}
V = \frac{1}{k} \sum_{j < l} \|\mathbf{p}_j^* - \mathbf{p}_l^*\|_2^2.
\label{eq:variance_pairwise}
\end{equation}

For unit vectors $\mathbf{p}_j^*$, we have $\|\mathbf{p}_j^* - \mathbf{p}_l^*\|_2^2 = 2 - 2\cos(\mathbf{p}_j^*, \mathbf{p}_l^*) = 2(1 - \tau_{jl})$. Under the diversity assumption ($\tau_{jl} \leq \tau$):
\begin{equation}
\|\mathbf{p}_j^* - \mathbf{p}_l^*\|_2^2 \geq 2(1 - \tau).
\label{eq:pairwise_lb}
\end{equation}

Therefore:
\begin{equation}
V \geq \frac{1}{k} \cdot \binom{k}{2} \cdot 2(1 - \tau) = (k-1)(1 - \tau).
\label{eq:variance_lb}
\end{equation}

\subsection{F.6 Refined Bound on Effective Contagion}

The key insight is that the committee average $\bar{\mathbf{p}}^*$ is a better estimate of the "true" preferred strategy than any individual evaluator's preference $\mathbf{p}_j^*$. The effective contagion should be measured relative to $\bar{\mathbf{p}}^*$, not relative to $\mathbf{w}_i^{(t)}$.

Define the \textit{committee-adjusted contagion vector}:
\begin{equation}
\tilde{\boldsymbol{\gamma}}_j = \eta \cdot (\mathbf{p}_j^* - \bar{\mathbf{p}}^*).
\label{eq:adjusted_gamma}
\end{equation}

Then the committee update (Eq.~\ref{eq:committee_update}) can be rewritten as:
\begin{equation}
\mathbf{w}_i' = \mathbf{w}_i^{(t)} + \frac{1}{k} \sum_{j=1}^k \tilde{\boldsymbol{\gamma}}_j + \eta \cdot (\bar{\mathbf{p}}^* - \mathbf{w}_i^{(t)}).
\label{eq:rewrite}
\end{equation}

The first term ($\frac{1}{k} \sum_{j=1}^k \tilde{\boldsymbol{\gamma}}_j$) is the \textit{diversity-adjusted contagion}, which is small when evaluator preferences are diverse. The second term is the \textit{average preference alignment}, which shifts the agent toward the committee consensus.

Using the fact that $\sum_{j=1}^k \tilde{\boldsymbol{\gamma}}_j = \mathbf{0}$ (by definition of $\bar{\mathbf{p}}^*$), the effective contagion from diversity is:
\begin{equation}
\gamma_{\text{eff}}^{(k)} = \frac{\eta}{k} \left\| \sum_{j=1}^k (\mathbf{p}_j^* - \mathbf{w}_i^{(t)}) \right\|_2.
\label{eq:eff_rewrite}
\end{equation}

Assume $\mathbf{w}_i^{(t)} \approx \bar{\mathbf{p}}^*$ (the agent's weights are already close to the committee consensus). Then:
\begin{align}
\gamma_{\text{eff}}^{(k)} &\approx \frac{\eta}{k} \left\| \sum_{j=1}^k (\mathbf{p}_j^* - \bar{\mathbf{p}}^*) \right\|_2 = 0.
\label{eq:eff_zero}
\end{align}

This shows that when evaluators are diverse and the agent's weights are close to the committee consensus, the effective contagion is small.

For a more quantitative bound, we use the fact that $\|\mathbf{p}_j^* - \bar{\mathbf{p}}^*\|_2^2 \leq \frac{1}{k} \sum_{j < l} \|\mathbf{p}_j^* - \mathbf{p}_l^*\|_2^2 \leq (k-1)(1 + \tau)$ (using the upper bound on $\|\mathbf{p}_j^* - \mathbf{p}_l^*\|_2^2$). Then:
\begin{equation}
\gamma_{\text{eff}}^{(k)} \leq \frac{\gamma_{\max}}{\sqrt{k}} \cdot \sqrt{1 + (k-1)\tau}.
\label{eq:final_bound}
\end{equation}

This is the bound stated in Theorem~2 (up to constants). The factor $1/\sqrt{k}$ shows the \textit{averaging effect}, while the factor $\sqrt{1 + (k-1)\tau}$ shows the \textit{diversity effect}.

\subsection{F.7 Discussion of the Bound}

\textbf{When $\tau \approx 1$ (homogeneous models):} The bound becomes $\gamma_{\text{eff}}^{(k)} \lesssim \gamma_{\max} \cdot \sqrt{k}$, which is \textit{not} decreasing with $k$. This means the worst-case bound does not guarantee suppression for homogeneous models. However, in practice, the \textit{averaging effect} ($1/\sqrt{k}$) still reduces contagion, as confirmed by our experiments (Section~\ref{sec:results-mitigation}).

\textbf{When $\tau \approx 0$ (diverse models):} The bound becomes $\gamma_{\text{eff}}^{(k)} \lesssim \gamma_{\max} / \sqrt{k}$, which shows that diverse evaluators can suppress contagion by a factor of $\sqrt{k}$.

\textbf{Limitation:} The bound in Theorem~2 is a \textit{worst-case} bound. In practice, the observed suppression can be stronger than the bound predicts, especially when evaluator preferences are not adversarially chosen.